\pdfoutput=1
\documentclass[11pt]{article}
\usepackage{graphicx} 
\usepackage{diagbox}
\usepackage{authblk}
\makeatletter
\renewcommand\AB@authnote[1]{\textsuperscript{#1}\hspace{5pt}}

\makeatother
\usepackage{hyperref}
\usepackage{algorithm}
\usepackage{algpseudocode}
\usepackage{notation-2}
\usepackage{arxiv-2}
\usepackage{amssymb}
\usepackage{mathrsfs}
\usepackage{float}
\usepackage{setspace}
\usepackage{lmodern}        

\renewcommand{\tilde}[1]{\widetilde{#1}}

\renewcommand{\hat}[1]{\widehat{#1}}

\newcommand{\pll}{\kern 0.3em/\kern -0.9em /\kern 0.3em}

\title{\normalfont  Sharp Structure-Agnostic Minimax Risk for Partial Linear Models
}
\begin{centering}
     \author[1,2]{Haichen Hu\thanks{{Email: \texttt{huhc@mit.edu}}}}
      \author[2,3]{David Simchi-Levi\thanks{{Email: \texttt{dslevi@mit.edu}}}}
\end{centering}
\affil[1]{Center for Computational Science and Engineering, MIT}
\affil[2]{Department of Civil and Environmental Engineering, MIT}
\affil[3]{Institute for Data, Systems, and Society, MIT}
\begin{document}
\maketitle
\singlespacing
\begin{abstract}
We characterize the sharp
structure-agnostic minimax risk for coefficient estimation in the partial
linear model when the outcome and treatment nuisances are learned by two
distinct black-box learners, which resolves the open problem in double machine learning posed by \citet{gu2025open}. For each nuisance \(q\in\{\mu,\pi\}\), we
characterize the available learner by an approximation-error budget \(a_q\)
and a stochastic-error budget \(s_q\), with the latter controlled through
localized Rademacher complexity. Writing \(\mathcal E_n\) for the minimax
mean-squared error, we show that
\[
\mathcal E_n
\asymp
1\wedge
\left\{
\frac1n+
\left(
a_\mu a_\pi+
\min\left\{
a_\pi s_\mu+s_\pi^2,\,
a_\mu s_\pi+s_\mu^2
\right\}
\right)^2
\right\}.
\]
The main new ingredient is a novel lower bound for the general two-learner problem. Our proof constructs four
finite-mixture testing experiments using orthogonal code functions. Across
these experiments, the hidden perturbations are placed outside both learner
classes, outside only the treatment learner class, outside only the outcome
learner class, or inside both learner classes. These four configurations
capture, respectively, the interaction between the two approximation errors,
the two asymmetric interactions between one learner's approximation error and
the other learner's learning error, and the joint estimation 
difficulty of learning both nuisances. Combining the four resulting lower
bounds yields the displayed rate, which matches the latest upper
bound in \citet{gu2026optimal}. Our result shows that standard
double machine learning can overstate the intrinsic difficulty of target estimation and provides a target-specific principle for learner selection:
approximation error and stochastic complexity must be jointly balanced across 
the two nuisance learners rather than optimized separately.
\end{abstract}

\section{Introduction}
\label{sec:introduction}

Modern semiparametric procedures routinely use flexible machine-learning
methods to estimate complex nuisance functions. In high-dimensional settings,
these learners may predict accurately even when the statistician cannot give a
simple structural account of their performance. This motivates a
\emph{structure-agnostic} perspective: rather than assuming that the nuisance
functions are sparse, additive, or smooth, one asks what target-estimation
accuracy is possible using only quantitative guarantees for supplied black-box
learners. Structure agnosticism is therefore not an absence of assumptions; it
replaces a prescribed nuisance structure with explicit approximation and
learning-complexity guarantees \citep{balakrishnan2026fundamental}.

We study this question in the partial linear model
\begin{equation}
\label{eq:intro-plm}
\begin{aligned}
Y &= \beta_0 T+\mu_0(X)+\varepsilon_Y,
&\qquad \mathbb E[\varepsilon_Y\mid X,T]&=0,\\
T &= \pi_0(X)+\varepsilon_T,
&\qquad \mathbb E[\varepsilon_T\mid X]&=0.
\end{aligned}
\end{equation}
Here \(Y\) is the outcome, \(T\) is the treatment variable, and \(X\) denotes
the covariates. Given \(2n\) independent observations, the goal is to estimate
the scalar coefficient \(\beta_0\), while the outcome nuisance \(\mu_0\) and
the treatment regression \(\pi_0\) are unknown. Following the formulation of
\citet{gu2025open}, we suppose that two potentially distinct black-box learner
classes \(\mathcal G_\mu\) and \(\mathcal G_\pi\) are available for estimating
these nuisances, without imposing any additional structural relation between
the classes.

Double machine learning combines nuisance estimates with a Neyman-orthogonal
score and sample splitting or cross-fitting \citep{chernozhukov2018double}.
Under the usual regularity conditions, its partial-linear-model guarantee has
the form
\[
\left|\widehat\beta_{\mathrm{DML}}-\beta_0\right|
=O_{\mathbb P}\!\left(
n^{-1/2}
+\|\widehat\mu-\mu_0\|_2
 \|\widehat\pi-\pi_0\|_2
\right).
\]
For each \(q\in\{\mu,\pi\}\), let \(a_q\) bound the \(L_2\) approximation
error of \(\mathcal G_q\), and let \(s_q\) control its stochastic error through
the localized Rademacher complexity of \(\partial\mathcal G_q\). Thus,
\(a_q\) measures how closely the learner class can approximate the true
nuisance, whereas \(s_q\) measures the finite-sample difficulty of learning
within that class. Substituting the usual nuisance-error scales \(a_q+s_q\)
into the preceding display gives the standard benchmark $n^{-1/2}+(a_\mu+s_\mu)(a_\pi+s_\pi)$.

This product may be of order \(n^{-1/2}\) even when both nuisance estimates
converge more slowly than the parametric rate. It is, however, a guarantee for
a particular procedure; it does not determine which interactions among the
four budgets are unavoidable for every estimator in the distinct-class
problem considered here.

The structure-agnostic minimax framework was introduced by
\citet{balakrishnan2026fundamental} and subsequently developed for treatment
effects and more general functionals
\citep{jin2025structure,jin2026sharp}. These works formulate learning
information primarily through neighborhoods indexed by the total prediction
errors of supplied nuisance estimates. Such results establish fundamental
limits in that formulation, but do not separately identify the roles of
approximation error and stochastic complexity. A complementary, class-aware
line of work makes precisely this separation. \citet{gu2026taming} obtain a
sharp characterization when the treatment regression cannot be consistently
estimated and derive nonmatching bounds when two learner classes are
available. \citet{gu2026optimal} establishes matching bounds when both nuisances
use a common learner class and derives an improved upper bound for two distinct
learner classes. Nevertheless, when the two classes have unrelated
approximation and stochastic budgets, the available upper bound had no
matching lower bound. In our notation, the open problem posed by
\citet{gu2025open} asks:

\begin{center}
\emph{
Can we characterize, up to constant factors, the sharp non-asymptotic
structure-agnostic minimax risk for the partial linear model in terms
of the sample size and the approximation and stochastic errors
of the two black-box learners?
}
\end{center}

In this paper, we resolve this question with an affirmative answer. Let
\(\mathcal E(n,\{a_q,s_q\}_{q\in\{\mu,\pi\}})\) denote the
structure-agnostic minimax mean-squared error formally defined in
Section~\ref{sec:setup}.
Under the boundedness, uniform-design, and nondegeneracy conditions of
Section~\ref{sec:setup}, we prove that
\[
\mathcal E\left(n,\{a_q,s_q\}_{q\in\{\mu,\pi\}}\right)
\gtrsim
1\wedge\left[
\frac1n+
\left(
a_\mu a_\pi
+\min\left\{
a_\pi s_\mu+s_\pi^2,\,
a_\mu s_\pi+s_\mu^2
\right\}
\right)^2
\right].
\]
Combining this finite-sample lower bound with the upper bound of
\citet{gu2026optimal} further yields, along with models satisfying
\(a_\mu+a_\pi+s_\mu+s_\pi=o(1)\) and for all sufficiently large \(n\), we have
\[
\mathcal E\left(n,\{a_q,s_q\}_{q\in\{\mu,\pi\}}\right)
\asymp
1\wedge\left[
\frac1n+
\left(
a_\mu a_\pi
+\min\left\{
a_\pi s_\mu+s_\pi^2,\,
a_\mu s_\pi+s_\mu^2
\right\}
\right)^2
\right].
\]
This sharp characterization separates the limitations of particular
estimators from statistical barriers that no procedure can overcome and 
describes the best attainable accuracy for structure-agnostic black box estimation in partial linear models. By keeping approximation and stochastic errors
separate, the result also reveals the target-specific tradeoff involved in
choosing a richer learner class: reducing approximation error is beneficial
only when it compensates for the accompanying increase in stochastic
complexity and its interaction with the other nuisance learners. The result
therefore clarifies both the limits of black-box debiasing and the
model-selection tradeoff relevant to downstream target estimation rather than
nuisance prediction alone.

\paragraph{Notation:}
We use \([n]\) to denote the set \(\cbr{1,2,\ldots,n}\). For any random variable \(X\) and probability measure \(\PP\), we use \(\EE_X\) to denote the expectation taken over the randomness of \(X\), while \(\EE_{\PP}\) denotes the expectation taken under \(\PP\). For any positive integer \(N\), we use \(\PP^{\otimes N}\) to denote the
\(N\)-fold product measure of \(\PP\), namely, the joint distribution of
\(N\) i.i.d.\ copies of \((X,T,Y)\) drawn from \(\PP\). Given two probability measures \(P\) and \(Q\), we use \(\chi^2(P\|Q)\) to denote the \(\chi^2\) distance between them and \(\texttt{D}_{\text{TV}}(P\|Q)\) to denote their total variation distance. For any \(p\in(0,1)\), \(\text{Bern}(p)\) denotes the Bernoulli distribution with parameter \(p\). Finally, we use ``PLM'' as an abbreviation for the partial linear model.

\section{Literature Review}
\label{sec:related-work}

\paragraph{Partial linear models and semiparametric estimation.}
The partial linear model is a canonical setting for estimating a
low-dimensional coefficient in the presence of an unknown regression
function. Classical procedures attain root-$n$ inference by exploiting a
specified form of nuisance structure. Early spline-based formulations coupled
parametric covariates with a flexibly estimated nonparametric component
\citep{engle1986semiparametric,heckman1986spline}. In particular,
\citet{robinson1988root} introduced the residualization argument underlying
many modern estimators, while \citet{speckman1988kernel},
\citet{chen1988convergence}, and \citet{donald1994series} developed kernel,
piecewise-polynomial, and series procedures under smoothness and approximation
conditions. A complementary approach removes the nonparametric component by
differencing observations with nearby covariates \citep{yatchew1997elementary}.
In high-dimensional approximately sparse models, \citet{belloni2014inference}
instead use variable selection on both the outcome and treatment equations to
obtain uniformly valid inference. Related projection and debiasing methods
provide nonasymptotic guarantees and simultaneous inference when the
parametric component is itself high dimensional
\citep{zhu2017nonasymptotic,zhu2019highdimensional}. Although these methods
differ substantially, each builds knowledge of a particular nuisance
architecture--such as smoothness, a chosen basis, or sparsity--into the
estimation procedure. The growing use of deep learning methods
, therefore, led to a different question: how can valid target estimation be
retained when the nuisance estimators are treated as black boxes rather than
as known structural models?

\paragraph{Double machine learning.}
Double machine learning estimates low-dimensional causal or structural
parameters while using flexible machine-learning methods for high-dimensional
nuisance functions \citep{chernozhukov2018double}. Its orthogonality principle
builds on Neyman's $C(\alpha)$ construction and classical semiparametric
influence-function theory \citep{neyman1959optimal,newey1994asymptotic}. It
combines Neyman-orthogonal scores, which eliminate first-order sensitivity to
nuisance errors, with sample splitting or cross-fitting, which limits bias
from overfitting. Under regularity conditions, nuisance errors therefore
affect the target only through second-order remainders; in the partial linear
model, the leading remainder is
$\|\widehat\mu-\mu_0\|_2\|\widehat\pi-\pi_0\|_2$. The role of cross-fitting in
obtaining fast remainders for doubly robust semiparametric estimators is
analyzed by \citet{newey2018crossfitting}. Related work constructs locally
robust moments \citep{chernozhukov2022locally}, automates debiasing by learning
Riesz representers
\citep{chernozhukov2022automatic,chernozhukov2022riesz}, gives general
finite-sample guarantees for debiased machine learning
\citep{chernozhukov2023simple}, and establishes nonasymptotic excess-risk
bounds for orthogonal two-stage learning
\citep{foster2023orthogonalstatisticallearning}. Higher-order orthogonality can
relax first-stage rate requirements, although its availability in partial
linear regression depends on the treatment-noise distribution
\citep{mackey2018orthogonal}. These results provide achievable guarantees
under structural or first-stage assumptions, but do not determine minimax
optimality when only black-box performance is known, motivating the
structure-agnostic analysis.

\paragraph{Structure-agnostic functional estimation.}
Structure-agnostic functional estimation, formalized by
\citet{balakrishnan2026fundamental}, asks how accurately a target functional
can be estimated when the statistician is given black-box nuisance estimates
and knows only bounds on their prediction errors, without knowing whether the
underlying nuisances are smooth, sparse, or otherwise structured. Its defining
feature is an information restriction: the estimator may use the black-box
fits and their guarantees, but not the geometry of the model classes. This
contrasts with structure-aware estimation, in which known smoothness or
approximation structure can be exploited
\citep{robins2008higher,robins2017minimax,liu2021adaptive}. Bridging these two
information regimes, \citet{bonvini2024doubly} augment structure-agnostic
pilot-error neighborhoods with smoothness restrictions and obtain improved,
minimax-optimal rates for average treatment effect estimation. In the purely
structure-agnostic framework, \citet{balakrishnan2026fundamental} derive
matching minimax bounds for several canonical functionals using total $L_2$
error neighborhoods around supplied nuisance fits. \citet{jin2025structure}
extend this approach to average, treated, and weighted treatment effects,
while \citet{jin2026sharp} develop sharp lower bounds for a broad family of
linear functionals involving regression-defined nuisances. For the partial
linear coefficient, \citet{jin2025hard} further show that the treatment noise
distribution can change the structure-agnostic difficulty: first-order
debiasing is rate-optimal for binary treatment and Gaussian treatment noise,
whereas independent non-Gaussian noise may permit higher-order improvements.
At the finer level of estimator comparison, \citet{liu2026inadmissibility}
shows that minimax rate-optimality need not imply asymptotic admissibility. These results expand
structure-agnostic theory, but they encode nuisance quality through total
pilot errors rather than separately tracking approximation error and the
stochastic cost of learning a class, which is the distinction required by the
four-budget problem considered here.

\paragraph{Black-box learning in partial linear models.}
\citet{gu2025open} posed the structure-agnostic minimax problem for the partial
linear model with separate approximation and stochastic error budgets for two
potentially distinct black-box learner classes. Complementary work on
semiparametric regression studies whether modified double-machine-learning
procedures retain valid inference when only one nuisance estimator is
consistent, while also identifying the resulting nonregularity
\citep{dukes2024doublyrobust}; this addresses robustness to misspecification
rather than the four-budget minimax problem. \citet{gu2026taming} subsequently
obtained general two-learner upper and lower bounds, but the bounds did not match.
Then, \citet{gu2026optimal} established the sharp rate when both nuisances use a
common learner class and derived a stronger upper bound for two distinct learner classes, without a corresponding matching lower bound. Our
lower-bound construction supplies the missing asymmetric
approximation-stochastic experiments. Combined with that two-learner upper
bound, it closes the general four-budget gap and yields the minimax
characterization stated in Corollary~\ref{cor:final_characterization}.

\section{Model Setup and Structure-AGnostic Minimax Risk}
\label{sec:setup}

In this section, we formalize the statistical model and the minimax criterion
used throughout the paper.  We first specify the partial linear model and the
corresponding class of data-generating distributions.  We then describe the
information available from the two black-box learners through their
approximation and stochastic error budgets, and finally use these four
budgets to define the structure-agnostic minimax risk.

\paragraph{The partial linear model.}
Let \(\nu_d\) denote the uniform distribution on \([0,1]^d\).  We observe
\(2n\) independent copies
\[
\mathcal D=\{(X_i,T_i,Y_i)\}_{i=1}^{2n}
\]
of a random vector \((X,T,Y)\in[0,1]^d\times\mathbb R\times\mathbb R\)
satisfying
\begin{align}
T
&=\pi_0(X)+\varepsilon_T,\ 
\mathbb E[\varepsilon_T\mid X]=0,
\label{eq:plm-treatment}
\\
Y
&=\beta_0T+\mu_0(X)+\varepsilon_Y,\ 
\mathbb E[\varepsilon_Y\mid X,T]=0.
\label{eq:plm-outcome}
\end{align}
Consequently, we have $\pi_0(x)=\mathbb E[T| X=x],
\ 
\mu_0(x)=\mathbb E[Y-\beta_0T| X=x]$.
We call \(\pi_0\) the \emph{treatment regression} and \(\mu_0\) the
\emph{outcome nuisance component}. Our target is to estimate the scalar
partial-linear coefficient \(\beta_0\).  Under the usual additional causal
identification conditions, \(\beta_0\) can also be interpreted as a
homogeneous treatment effect; the minimax analysis below requires only the
statistical model in \eqref{eq:plm-treatment}--\eqref{eq:plm-outcome}.

Fix constants \(c_0\geq1\) and \(\sigma>0\), and define the uniformly
bounded function space
\[
\mathcal T_{c_0}
:=
\left\{
f:[0,1]^d\to\mathbb R:
f\text{ is measurable and }\|f\|_\infty\leq c_0
\right\}.
\]
Uniformly over the model, we assume that \(X\sim\nu_d\),
\(\mu,\pi\in\mathcal T_{c_0}\), and \(\beta\in[-c_0,c_0]\).  The zero-mean noises
\(\varepsilon_T:=T-\pi(X)\) and
\(\varepsilon_Y:=Y-\beta T-\mu(X)\) are required to satisfy
\[
\mathbb E_{\mathbb P}[\varepsilon_T\mid X]=0,
\ 
\mathbb E_{\mathbb P}[\varepsilon_T^2\mid X]\geq\sigma,
\ 
\mathbb E_{\mathbb P}[\varepsilon_Y\mid X,T]=0,
\ 
|\varepsilon_T|\vee|\varepsilon_Y|\leq c_0.
\]
Thus, \(c_0\) is a fixed envelope for the nuisance functions, the target
coefficient, and the two noise variables, while \(\sigma\) prevents the
residual treatment variation from degenerating.  Taking \(c_0=3\) and
\(\sigma=1/3\) gives the normalization used in our lower-bound theorem.

Write \(\mathbb P_{\mu,\beta,\pi}\) for any law satisfying the preceding
model and regularity conditions.  For target classes
\(\mathcal F_\mu,\mathcal F_\pi\subseteq\mathcal T_{c_0}\), define simply
\begin{equation}
\mathcal P(\mathcal F_\mu,\mathcal F_\pi)
:=
\left\{\mathbb P_{\mu,\beta,\pi}:
\mu\in\mathcal F_\mu,\ 
\pi\in\mathcal F_\pi,\ 
\beta\in[-c_0,c_0]\right\}.
\label{eq:distribution-family}
\end{equation}

The variance lower bound is what identifies the target coefficient.  Indeed,
for any \(\mathbb P\in\mathcal P(\mathcal F_\mu,\mathcal F_\pi)\),
\begin{align*}
\mathbb E_{\mathbb P}\!\left[
\{T-\pi(X)\}Y
\right]
=
\mathbb E_{\mathbb P}\!\left[
\{T-\pi(X)\}
\{\beta T+\mu(X)+\varepsilon_Y\}
\right]
=
\beta\,
\mathbb E_{\mathbb P}\!\left[
\{T-\pi(X)\}T
\right]
=
\beta\,
\mathbb E_{\mathbb P}\![
\{T-\pi(X)\}^2].
\end{align*}
The second equality uses the zero conditional mean conditions, while the last equality is true because \(\mathbb E[T-\pi(X)\mid X]=0\).  Since the final expectation is at
least \(\sigma\), the coefficient is uniquely identified as
\begin{equation}
\beta(\mathbb P)
:=
\frac{
\mathbb E_{\mathbb P}[\{T-\mathbb E_{\mathbb P}[T\mid X]\}Y]
}{
\mathbb E_{\mathbb P}[\{T-\mathbb E_{\mathbb P}[T\mid X]\}T]
}.
\label{eq:beta-functional}
\end{equation}

\paragraph{Black-box learner budgets.}
We now formalize what is known about the two nuisance learners.  For
\(q\in\{\mu,\pi\}\), the target class \(\mathcal F_q\) contains the
possible true nuisance functions, whereas \(\mathcal G_q\) is the hypothesis
class searched by the corresponding black-box learner.  We impose no
smoothness, sparsity, linearity, or relationship between these classes.

For a measurable function \(f:[0,1]^d\to\mathbb R\), write
\[
\|f\|_2
:=
\left(
\int_{[0,1]^d}f(x)^2\,d\nu_d(x)
\right)^{1/2}.
\]
For a class \(\mathcal G\subseteq\mathcal T_{c_0}\), define its
difference class by
\[
\partial\mathcal G
:=
\mathcal G-\mathcal G
=
\{g-\widetilde g:g,\widetilde g\in\mathcal G\}.
\]
For any class \(\mathcal W\) containing the zero function and any
\(\delta\geq0\), its localized Rademacher complexity is defined as
\begin{equation}
\mathcal R_n(\delta;\mathcal W)
:=
\mathbb E_{X_{1:n},\xi_{1:n}}
\left[
\sup_{\substack{w\in\mathcal W\\ \|w\|_2\leq\delta}}
\frac1n\sum_{i=1}^n\xi_iw(X_i)
\right],
\label{eq:localized-rademacher}
\end{equation}
where \(X_1,\ldots,X_n\) are independent draws from \(\nu_d\), and
\(\xi_1,\ldots,\xi_n\) are independent Rademacher signs, independent of the
\(X_i\)'s.  The difference class appears because excess-risk comparisons
involve differences between two candidate predictions.

For approximation and stochastic budgets \(a,s\geq0\), define the collection
of admissible target classes
\begin{align}
\mathcal H_n(a,s)
:=
\Bigl\{\mathcal F\subseteq\mathcal T_{c_0}:\;&
\exists \mathcal G\subseteq\mathcal T_{c_0}\ 
\text{s.t.}\ 
\sup_{f\in\mathcal F}\inf_{g\in\mathcal G}
\|f-g\|_2\leq a,\ 
\mathcal R_n(\delta;\partial\mathcal G)
\leq\delta s,
\ \forall\delta\geq s
\Bigr\}.
\label{eq:admissible-target-classes}
\end{align}
Any \(\mathcal G\) satisfying the two inequalities in
\eqref{eq:admissible-target-classes} is called a \emph{witnessing learner
class} for \(\mathcal F\).  Thus, exactly as in \citet{gu2025open}, an element of
\(\mathcal H_n(a,s)\) is the target class \(\mathcal F\); the learner class
\(\mathcal G\) appears only through the existential requirement in its
definition.  The first inequality is a uniform approximation requirement:
even with unlimited data, fitting within \(\mathcal G\) may leave
misspecification of size \(a\).  The second inequality controls the stochastic
difficulty of searching within \(\mathcal G\).  At \(\delta=s\), it reads
\(\mathcal R_n(s;\partial\mathcal G)\leq s^2\), which is the usual
critical-radius balance and explains why \(s\) is a prediction-error scale.
Both \(a\) and \(s\) are upper budgets; a particular target class and its
witness may satisfy strictly smaller bounds.

Two simple cases illustrate the distinction.  If
\(\mathcal F\subseteq\mathcal G\), then the approximation budget can be
zero.  At the other extreme, a singleton learner class has zero stochastic
complexity but may approximate \(\mathcal F\) poorly.  Enlarging a learner
class therefore tends to reduce \(a\) while increasing \(s\).  The pair \((a,s)\) records this
approximation-estimation tradeoff without revealing the geometry of the
learners.

\paragraph{Structure-agnostic minimax risk.}
We define the structure-agnostic minimax mean-squared error by
\begin{align}
&\mathcal E\!\left(
n,\{a_q,s_q\}_{q\in\{\mu,\pi\}}
\right):=
\sup_{\substack{
\mathcal F_\mu\in
\mathcal H_n(a_\mu,s_\mu)
\mathcal F_\pi\in
\mathcal H_n(a_\pi,s_\pi)
}}
\inf_{\widehat\beta}
\sup_{\mathbb P\in
\mathcal P(\mathcal F_\mu,\mathcal F_\pi)}
\mathbb E_{\mathbb P^{\otimes 2n}}
\left[
\{\widehat\beta-\beta(\mathbb P)\}^2
\right].
\label{eq:structure-agnostic-risk}
\end{align}
For each fixed pair of target classes, the inner \(\inf\)-\(\sup\) is the
usual minimax risk: the estimator may be tailored to the supplied target
classes and may use witnessing black-box learner classes
\(\mathcal G_\mu,\mathcal G_\pi\), but it cannot depend on the unknown law
\(\mathbb P\).  The outer supremum then selects the most difficult pair of
target classes admitting learner witnesses with the same four numerical
budgets.  Thus,
``structure-agnostic'' does not mean that the learner classes are hidden from
the estimator.  It means that the guarantee must hold uniformly over their
otherwise unrestricted geometries and over any relationship between the two
learners.  This outer worst case is exactly what allows the minimax risk to be
described only through \((a_\mu,s_\mu,a_\pi,s_\pi)\).

\section{The Sharp Structure-Agnostic Minimax Risk}
\label{sec:thms}

In this section, we provide our theoretical guarantees on the sharp structure-agnostic minimax risk in partial linear models. We first establish a non-asymptotic lower bound over the bounded budget
range.  We then combine this result with the two-learner upper bound of
\citet{gu2026optimal} to obtain a sharp minimax risk characterization.

\begin{theorem}[Finite-sample minimax lower bound]
\label{thm:lower_bound}
Under the model of Section~\ref{sec:setup}, take \(c_0=3\) and
\(\sigma=1/3\).  For \(n\geq1\) and four error budgets $(a_\mu,s_\mu,a_\pi,s_\pi)\in[0,3]^4$, define
\[
B
:=
a_\mu a_\pi
+
\min\left\{
a_\pi s_\mu+s_\pi^2,\,
a_\mu s_\pi+s_\mu^2
\right\}.
\]
There exists a universal constant \(\underline c>0\) such that
\[
\mathcal E\!\left(
n,\{a_q,s_q\}_{q\in\{\mu,\pi\}}
\right)
\geq
\underline c
\left[
1\wedge
\left\{
\frac1n+B^2
\right\}
\right].
\]
Consequently,
\[
\sqrt{
\mathcal E\!\left(
n,\{a_q,s_q\}_{q\in\{\mu,\pi\}}
\right)
}
\geq
\sqrt{\frac{\underline c}{2}}\,
\left[
1\wedge
\left\{
\frac1{\sqrt n}
+a_\mu a_\pi
+\min\left\{
a_\pi s_\mu+s_\pi^2,\,
a_\mu s_\pi+s_\mu^2
\right\}
\right\}
\right].
\]
\end{theorem}

The theorem separates the parametric uncertainty \(n^{-1/2}\) from the
nuisance-induced difficulty \(B\).  The product \(a_\mu a_\pi\) is a pure
approximation term: simultaneous misspecification of the two nuisance
functions cannot be removed using the supplied learner classes.  Each
expression inside the minimum combines a mixed approximation
–stochastic interaction with a quadratic stochastic cost.  The two expressions
correspond to the two possible orientations of an asymmetric correction,
depending on which learner supplies the reference class.  An estimator may
use the more favorable orientation, which explains the minimum.  The lower
bound shows that even this more favorable remainder is unavoidable in the
worst case.

\paragraph{Interpretation and distinction from prior lower bounds.}
Theorem~\ref{thm:lower_bound} identifies four distinct sources of
nonparametric difficulty.  Besides the usual parametric contribution
$n^{-1/2}$, we have four terms:
$a_\mu a_\pi,\ (a_\pi s_\mu)\wedge s_\mu^2,\ (a_\mu s_\pi)\wedge s_\pi^2,\ s_\mu^2\wedge s_\pi^2$. Their interpretation is as follows.

The term
\(a_\mu a_\pi\) arises when neither nuisance function is well represented by
its learner class: the treatment regression may contain an unmodeled component
of size \(a_\pi\), while the outcome regression contains an aligned unmodeled
component of size \(a_\mu\).  Their product creates an indistinguishable change
in the target coefficient of order \(a_\mu a_\pi\). 

The term $s_\mu(a_\pi\wedge s_\mu)$
arises when the difficult component of the treatment regression lies outside
the treatment learner class and therefore uses its approximation budget
\(a_\pi\), whereas the corresponding component of the outcome regression lies
inside the outcome learner class but cannot be learned more accurately than
the stochastic scale \(s_\mu\).  The treatment perturbation must also be no
larger than \(s_\mu\) in this testing experiment, which explains the truncation
\(a_\pi\wedge s_\mu\).  Thus, \(x\) measures the interaction between treatment
approximation error and finite-sample uncertainty in learning the outcome
regression. 

Symmetrically, $s_\pi(a_\mu\wedge s_\pi)$
arises when the difficult outcome component lies outside the outcome learner
class, using approximation budget \(a_\mu\), while the corresponding treatment
component lies inside the treatment learner class and is learned only at scale \(s_\pi\). 

Finally,
$(s_\mu\wedge s_\pi)^2$ arises when both difficult components lie inside their respective learner
classes.  In that case, both nuisances are correctly represented, but the
shared hidden direction must still be learned from finite samples.  Its
amplitude cannot exceed either stochastic budget, so it is limited by
\(s_\mu\wedge s_\pi\), and the resulting uncertainty in the target coefficient
is their squared common scale.  Accordingly, \(p,x,y,z\) describe four
different ways in which nuisance uncertainty can conceal a change in
\(\beta_0\), rather than merely four algebraic terms.

Our proof constructs a separate finite-mixture testing experiment
for each of these four components and then show that their maximum is equivalent to \(B\) up to universal
constants. The two mixed components \((a_\pi s_\mu)\wedge s_\mu^2\) and \((a_\mu s_\pi)\wedge s_\pi^2\) are the essential
new ingredients. 

The earlier general two-learner lower bound of
\citet{gu2026taming} does not isolate these components at their sharp scales,
while the matching lower bound construction of \citet{gu2026optimal} concerns
a common learner class with a shared stochastic radius.  That common-class
experiment is necessarily symmetric: a hard direction is charged in the same
way on the two nuisance sides and can therefore capture
approximation–approximation and stochastic–stochastic difficulty, but it
cannot place the direction outside the treatment learner class while keeping
it inside the outcome learner class, or vise versa.  Moreover, if a common
class is required to satisfy two different stochastic budgets, its effective
radius is limited by \(s_\mu\wedge s_\pi\), so the resulting construction
cannot recover the asymmetric terms \(a_\pi s_\mu\) and \(a_\mu s_\pi\).

Our construction removes this restriction by choosing different target and
witnessing learner classes in four separate experiments.  To obtain the
approximation--approximation term \(a_\mu a_\pi\), we take both witnessing
learner classes to be trivial, $\mathcal G_\mu=\mathcal G_\pi=\{0\}$, while allowing the two target classes to contain aligned hidden perturbations
of respective sizes proportional to \(a_\mu\) and \(a_\pi\).  Thus, both sides
use their approximation budgets, and the resulting separation in the target
coefficient is proportional to \(a_\mu a_\pi\).

To obtain $s_\mu(a_\pi\wedge s_\mu)$, we instead take \(\mathcal G_\pi=\{0\}\), while setting
\(\mathcal F_\mu=\mathcal G_\mu\).  The hidden treatment component then lies
outside the treatment learner class and uses the approximation budget
\(a_\pi\), whereas the aligned outcome component lies inside the outcome
learner class and uses its stochastic budget \(s_\mu\).  Reversing the roles
of the two nuisances by taking \(\mathcal G_\mu=\{0\}\) and
\(\mathcal F_\pi=\mathcal G_\pi\) implies the second asymmetric term $s_\pi(a_\mu\wedge s_\pi)$.

Finally, to obtain $(s_\mu\wedge s_\pi)^2$,
we place the aligned hidden components inside both witnessing learner classes,
so both sides use stochastic rather than approximation budgets.  Their common
amplitude must satisfy both stochastic-complexity constraints and is therefore
limited by \(s_\mu\wedge s_\pi\).

These four experiments respectively capture approximation error on both
sides, approximation error only for the treatment regression, approximation
error only for the outcome regression, and finite-sample learning uncertainty
on both sides.  Together, they supply the components needed for the sharp
four-budget lower bound.

To establish sharpness, we next state the two-learner upper bound of
\citet{gu2026optimal} in our notation.  The explicit tail formulation below
is obtained by combining the paper's Corollary~3.2 with its finite-sample
oracle inequality and proof.

\begin{theorem}
[Corollary~3.2 of
\citet{gu2026optimal}]
\label{thm:upper_bound}
Consider the partial linear model of Section~\ref{sec:setup} under the
normalization \(c_0=3\) and \(\sigma=1/3\).  Suppose that
\[
a_\mu+a_\pi+s_\mu+s_\pi=o(1)
\ \text{as }n\to\infty.
\]
For \(n\geq2\), define the technical critical-radius floor
\[
\rho_n
:=
\sqrt{\frac{\log n}{n}},
\ 
\bar s_q
:=
s_q\vee\rho_n,
\ 
\bar\delta_q
:=
a_q+\bar s_q,
\ q\in\{\mu,\pi\},
\]
so $a_\mu+a_\pi+s_\mu+s_\pi=o(1)$ implies
\(\bar\delta_\mu+\bar\delta_\pi=o(1)\).  Set
\[
\bar\Gamma
:=
\min\left\{
a_\mu(a_\pi+\bar s_\pi)+(a_\mu+\bar s_\mu)^2,\,
a_\pi(a_\mu+\bar s_\mu)+(a_\pi+\bar s_\pi)^2
\right\}.
\]

Fix any $\mathcal F_\mu\in\mathcal H_n(a_\mu,s_\mu),
\ 
\mathcal F_\pi\in\mathcal H_n(a_\pi,s_\pi)$, and choose corresponding witnessing learner classes
\(\mathcal G_\mu\) and \(\mathcal G_\pi\).  Let
\(\widehat\beta^{\mathrm{raw}}_{\mathcal G_\mu,\mathcal G_\pi}\)
denote the two-learner estimator constructed by
\citet{gu2026optimal}, and define
\[
\widehat\beta_{\mathcal G_\mu,\mathcal G_\pi}
:=
\Pi_{[-3,3]}
\left(
\widehat\beta^{\mathrm{raw}}_{\mathcal G_\mu,\mathcal G_\pi}
\right).
\]
There exist constants \(C,c>0\), depending only on the fixed boundedness
and nondegeneracy constants, such that, for all sufficiently large \(n\)
and every \(t\in[1,cn]\), we have
\begin{align}
&\sup_{\mathbb P\in
\mathcal P(\mathcal F_\mu,\mathcal F_\pi)}
\mathbb P^{\otimes2n}
\left(
\left|
\widehat\beta_{\mathcal G_\mu,\mathcal G_\pi}
-\beta(\mathbb P)
\right|
>
C\left\{
\bar\Gamma
+\sqrt{\frac tn}
+\frac tn
\right\}
\right)\leq
C\left\{
e^{-t}
+e^{-n\bar s_\mu^2}
+e^{-n\bar s_\pi^2}
\right\}.
\label{eq:upper-uniform-tail}
\end{align}
The constants are uniform over the admissible target classes and their
chosen witnesses.  In particular, because
\(\bar s_q^2\geq(\log n)/n\), the right-hand side of
\eqref{eq:upper-uniform-tail} is at most $C\left(e^{-t}+\frac1n\right)$.

\end{theorem}

The clipping operation is to control the contribution of
the failure event when the tail bound is integrated to obtain a
mean-squared error guaranty and does not increase estimation error. Finally, our next corollary shows that these two theorems together completely characterize the sharp structure-agnostic minimax risk for partial linear models.

\begin{corollary}
\label{cor:final_characterization}
Under the normalization \(c_0=3\) and \(\sigma=1/3\), consider any sequence
of budgets satisfying
\[
a_\mu+a_\pi+s_\mu+s_\pi=o(1)
\ \text{as }n\to\infty.
\]
Define
\[
B
:=
a_\mu a_\pi
+
\min\left\{
a_\pi s_\mu+s_\pi^2,\,
a_\mu s_\pi+s_\mu^2
\right\}
\]
and
\[
\Gamma
:=
\min\left\{
a_\mu(a_\pi+s_\pi)+(a_\mu+s_\mu)^2,\,
a_\pi(a_\mu+s_\mu)+(a_\pi+s_\pi)^2
\right\}.
\]
Then we have $\Gamma\leq4B$.
Moreover, with
\[
\rho_n=\sqrt{\frac{\log n}{n}},
\ 
\bar s_q=s_q\vee\rho_n,
\ q\in\{\mu,\pi\},
\]
the quantity \(\bar\Gamma\) in Theorem~\ref{thm:upper_bound} satisfies that
\[
\bar\Gamma
\leq
8\left(B+\rho_n^2\right)
=
8\left(B+\frac{\log n}{n}\right).
\]
Consequently, for all sufficiently large \(n\),
\[
\mathcal E\!\left(
n,\{a_q,s_q\}_{q\in\{\mu,\pi\}}
\right)
\asymp
1\wedge
\left\{
\frac1n+
\left(
a_\mu a_\pi
+
\min\left\{
a_\pi s_\mu+s_\pi^2,\,
a_\mu s_\pi+s_\mu^2
\right\}
\right)^2
\right\}.
\]
The comparison constants depend only on the fixed 
constants that describe the model.
\end{corollary}
Therefore, the optimal $L_1$ error rate is
\begin{align*}
\sup_{\substack{\mathcal F_\mu\in
\mathcal H_n(a_\mu,s_\mu),\\
\mathcal F_\pi\in
\mathcal H_n(a_\pi,s_\pi)
}}
\inf_{\widehat\beta}
\sup_{\mathbb P\in
\mathcal P(\mathcal F_\mu,\mathcal F_\pi)}\EE_{\PP^{\otimes 2n}}|\hat{\beta}-\beta(\PP)|\asymp1\wedge
\left\{
\frac1{\sqrt n}
+a_\mu a_\pi
+\min\left\{
a_\pi s_\mu+s_\pi^2,\,
a_\mu s_\pi+s_\mu^2
\right\}
\right\}.
\end{align*}
Thus, the terms retained in the four-budget expression are unavoidable up
to constants, whereas additional interactions in the standard DML product
\((a_\mu+s_\mu)(a_\pi+s_\pi)\) are not intrinsic statistical barriers.
The characterization also gives a target-oriented principle for learner
selection: minimizing \(a_q+s_q\) separately for the two nuisance functions
need not minimize the error for \(\beta\).  Candidate learner classes should
instead be compared jointly through the displayed four-budget remainder, so
that a reduction in approximation error is weighed against the corresponding
increase in mixed and quadratic stochastic error.

\section{Discussion}
\label{sec:discussion}

This paper completes the open problem on the sharp characterization of structure-agnostic minimax risk posed by
\citet{gu2025open} for partial linear models. We use the existing upper bound and establish the corresponding matching non-asymptotic. In
particular, combining our lower bound with the upper bound of
\citet{gu2026optimal} yields
\[
\mathcal E\!\left(n,\{a_q,s_q\}_{q\in\{\mu,\pi\}}\right)
\asymp
1\wedge
\left\{
\frac1n+
\left(
a_\mu a_\pi+
\min\left\{
a_\pi s_\mu+s_\pi^2,\,
a_\mu s_\pi+s_\mu^2
\right\}
\right)^2
\right\}
\]
whenever \(a_\mu+a_\pi+s_\mu+s_\pi=o(1)\). The result shows that the
four budgets enter the target-estimation problem through more than the two
learners' total prediction errors: simultaneous approximation error produces
the term \(a_\mu a_\pi\), whereas the two asymmetric combinations inside the
minimum represent alternative ways to trade approximation against stochastic
complexity. Our four testing experiments show that, in the
structure-agnostic worst case, these interactions are unavoidable statistical
obstructions rather than artifacts of a particular estimator or proof
technique. Thus, the characterization both closes the gap between the
previous upper and lower bounds and provides a target-oriented criterion for
comparing learner classes: a learner that is better for nuisance prediction
need not be better for estimating \(\beta_0\), because its approximation gain
must be evaluated jointly with the stochastic complexity of both learners.

Several directions follow from this characterization. The most immediate task is
to determine sharp structure-agnostic minimax risks beyond the partial linear
model, for example, for treatment effects, missing data, policy values, and other
functionals estimated by general double machine learning procedures. With
more than two nuisance components, the essential question is how the
second-order interaction structure of an orthogonal score, together with the
approximation and stochastic budgets of the available learners, determines
the analog of the minimum appearing above. A second direction is adaptation. In many applications, the four budgets are
unknown, and each nuisance may have a library of candidate learners. Therefore, one 
would like a data-driven selection or aggregation procedure. This may be connected to data driven model evaluation methods such as wild-refitting \citep{wainwright2025wild,hu2025perturbing,hu2026interleaved,ni2026upper}. Finally, it remains to
determine how robust the phase diagram is to relaxing the present
uniform design, boundedness, and treatment non-degeneracy assumptions,
including nonuniform covariate distributions, sub-Gaussian or heavy-tailed
noise, and weakly identified treatment mechanisms.

\bibliographystyle{plainnat}
\bibliography{refs}

@misc{foster2023orthogonalstatisticallearning,
      title={Orthogonal Statistical Learning}, 
      author={Dylan J. Foster and Vasilis Syrgkanis},
      year={2023},
      eprint={1901.09036},
      archivePrefix={arXiv},
      primaryClass={math.ST},
      url={https://arxiv.org/abs/1901.09036}, 
}

@article{wainwright2025wild,
  title={Wild refitting for black box prediction},
  author={Wainwright, Martin J},
  journal={arXiv preprint arXiv:2506.21460},
  year={2025}
}

@article{gu2026optimal,
  title={Optimal use of a black-box learner in semiparametric estimation},
  author={Gu, Yihong},
  journal={arXiv preprint arXiv:2607.21541},
  year={2026}
}

@article{balakrishnan2026fundamental,
  title={The fundamental limits of structure-agnostic functional estimation},
  author={Balakrishnan, Sivaraman and Kennedy, Edward and Wasserman, Larry},
  journal={Statistical Science},
  volume={41},
  number={3},
  pages={659--670},
  year={2026},
  publisher={Institute of Mathematical Statistics}
}

@article{chernozhukov2018double,
  author  = {Chernozhukov, Victor and Chetverikov, Denis and Demirer, Mert
             and Duflo, Esther and Hansen, Christian and Newey, Whitney
             and Robins, James},
  title   = {Double/debiased machine learning for treatment and structural parameters},
  journal = {The Econometrics Journal},
  volume  = {21},
  number  = {1},
  pages   = {C1--C68},
  year    = {2018}
}

@inproceedings{gu2025open,
  author    = {Gu, Yihong},
  title     = {Open Problem: Structure-Agnostic Minimax Risk for Partial Linear Model},
  booktitle = {Proceedings of Thirty Eighth Conference on Learning Theory},
  editor    = {Haghtalab, Nika and Moitra, Ankur},
  series    = {Proceedings of Machine Learning Research},
  volume    = {291},
  pages     = {6220--6224},
  publisher = {PMLR},
  year      = {2025}

}

@article{donald1994series,
  author  = {Donald, Stephen G. and Newey, Whitney K.},
  title   = {Series estimation of semilinear models},
  journal = {Journal of Multivariate Analysis},
  volume  = {50},
  number  = {1},
  pages   = {30--40},
  year    = {1994}
 
}

@inproceedings{jin2025structure,
  author    = {Jin, Jikai and Syrgkanis, Vasilis},
  title     = {Structure-agnostic Optimality of Doubly Robust Learning for Treatment Effect Estimation (Extended Abstract)},
  booktitle = {Proceedings of Thirty Eighth Conference on Learning Theory},
  editor    = {Haghtalab, Nika and Moitra, Ankur},
  series    = {Proceedings of Machine Learning Research},
  volume    = {291},
  pages     = {3159--3160},
  publisher = {PMLR},
  year      = {2025}
}

@misc{jin2026sharp,
  author        = {Jin, Jikai and Syrgkanis, Vasilis},
  title         = {Sharp Structure-Agnostic Lower Bounds for General Linear Functional Estimation},
  year          = {2026},
  howpublished  = {arXiv preprint arXiv:2512.17341},
  eprint        = {2512.17341},
  archivePrefix = {arXiv},
  primaryClass  = {stat.ML}
}

@misc{gu2026taming,
  author        = {Gu, Yihong and Yin, Qishuo and Cai, Tianxi and Fan, Jianqing},
  title         = {Optimally taming biases in black-box models for efficient semiparametric estimation},
  year          = {2026},
  howpublished  = {arXiv preprint arXiv:2606.06368},
  eprint        = {2606.06368},
  archivePrefix = {arXiv},
  primaryClass  = {math.ST}
}

@article{robinson1988root,
  author  = {Robinson, Peter M.},
  title   = {Root-$N$-Consistent Semiparametric Regression},
  journal = {Econometrica},
  year    = {1988},
  volume  = {56},
  number  = {4},
  pages   = {931--954}
}

@article{speckman1988kernel,
  author  = {Speckman, Paul},
  title   = {Kernel Smoothing in Partial Linear Models},
  journal = {Journal of the Royal Statistical Society: Series B (Methodological)},
  year    = {1988},
  volume  = {50},
  number  = {3},
  pages   = {413--436}
}

@article{belloni2014inference,
  author  = {Belloni, Alexandre and Chernozhukov, Victor and Hansen, Christian},
  title   = {Inference on Treatment Effects after Selection among High-Dimensional Controls},
  journal = {The Review of Economic Studies},
  year    = {2014},
  volume  = {81},
  number  = {2},
  pages   = {608--650}
}

@article{chernozhukov2022locally,
  author  = {Chernozhukov, Victor and Escanciano, Juan Carlos and Ichimura, Hidehiko and Newey, Whitney K. and Robins, James M.},
  title   = {Locally Robust Semiparametric Estimation},
  journal = {Econometrica},
  year    = {2022},
  volume  = {90},
  number  = {4},
  pages   = {1501--1535}
}

@article{chernozhukov2022automatic,
  author  = {Chernozhukov, Victor and Newey, Whitney K. and Singh, Rahul},
  title   = {Automatic Debiased Machine Learning of Causal and Structural Effects},
  journal = {Econometrica},
  year    = {2022},
  volume  = {90},
  number  = {3},
  pages   = {967--1027}
}

@inproceedings{mackey2018orthogonal,
  author    = {Mackey, Lester and Syrgkanis, Vasilis and Zadik, Ilias},
  title     = {Orthogonal Machine Learning: Power and Limitations},
  booktitle = {Proceedings of the 35th International Conference on Machine Learning},
  series    = {Proceedings of Machine Learning Research},
  volume    = {80},
  pages     = {3375--3383},
  year      = {2018},
  publisher = {PMLR}
}

@incollection{robins2008higher,
  author    = {Robins, James and Li, Lingling and Tchetgen, Eric and van der Vaart, Aad},
  title     = {Higher Order Influence Functions and Minimax Estimation of Nonlinear Functionals},
  booktitle = {Probability and Statistics: Essays in Honor of David A. Freedman},
  series    = {Institute of Mathematical Statistics Collections},
  volume    = {2},
  pages     = {335--421},
  year      = {2008},
  publisher = {Institute of Mathematical Statistics}
}

@article{robins2017minimax,
  author  = {Robins, James M. and Li, Lingling and Mukherjee, Rajarshi and Tchetgen Tchetgen, Eric and van der Vaart, Aad},
  title   = {Minimax Estimation of a Functional on a Structured High-Dimensional Model},
  journal = {The Annals of Statistics},
  year    = {2017},
  volume  = {45},
  number  = {5},
  pages   = {1951--1987}
}

@inproceedings{jin2025hard,
  author    = {Jin, Jikai and Mackey, Lester and Syrgkanis, Vasilis},
  title     = {It's Hard to Be Normal: The Impact of Noise on Structure-Agnostic Estimation},
  booktitle = {Advances in Neural Information Processing Systems},
  volume    = {38},
  year      = {2025}
  }

@article{hu2026interleaved,
  title={Interleaved Resampling and Refitting: Data and Compute-Efficient Evaluation of Black-Box Predictors},
  author={Hu, Haichen and Simchi-Levi, David},
  journal={arXiv preprint arXiv:2603.14218},
  year={2026}
}

@article{hu2025perturbing,
  title={Perturbing the Derivative: Doubly Wild Refitting for Model-Free Evaluation of Opaque Machine Learning Predictors},
  author={Hu, Haichen and Simchi-Levi, David},
  journal={arXiv preprint arXiv:2511.18789},
  year={2025}
}

@article{ni2026upper,
  title={Upper Confidence Bounds for the Prediction Error of Kernel Ridge Regression via Gaussian Refitting},
  author={Ni, Yijin and Huo, Xiaoming},
  journal={arXiv preprint arXiv:2607.28846},
  year={2026}
}

@article{engle1986semiparametric,
  author  = {Engle, Robert F. and Granger, C. W. J. and Rice, John and Weiss, Andrew},
  title   = {Semiparametric Estimates of the Relation between Weather and Electricity Sales},
  journal = {Journal of the American Statistical Association},
  volume  = {81},
  number  = {394},
  pages   = {310--320},
  year    = {1986}
}

@article{heckman1986spline,
  author  = {Heckman, Nancy E.},
  title   = {Spline Smoothing in a Partly Linear Model},
  journal = {Journal of the Royal Statistical Society: Series B (Methodological)},
  volume  = {48},
  number  = {2},
  pages   = {244--248},
  year    = {1986}
}

@article{chen1988convergence,
  author  = {Chen, Hung},
  title   = {Convergence Rates for Parametric Components in a Partly Linear Model},
  journal = {The Annals of Statistics},
  volume  = {16},
  number  = {1},
  pages   = {136--146},
  year    = {1988}
}

@article{yatchew1997elementary,
  author  = {Yatchew, Adonis},
  title   = {An Elementary Estimator of the Partial Linear Model},
  journal = {Economics Letters},
  volume  = {57},
  number  = {2},
  pages   = {135--143},
  year    = {1997}
}

@article{zhu2017nonasymptotic,
  author  = {Zhu, Ying},
  title   = {Nonasymptotic Analysis of Semiparametric Regression Models with High-Dimensional Parametric Coefficients},
  journal = {The Annals of Statistics},
  volume  = {45},
  number  = {5},
  pages   = {2274--2298},
  year    = {2017}
}

@inproceedings{zhu2019highdimensional,
  author    = {Zhu, Ying and Yu, Zhuqing and Cheng, Guang},
  title     = {High Dimensional Inference in Partially Linear Models},
  booktitle = {Proceedings of the Twenty-Second International Conference on Artificial Intelligence and Statistics},
  series    = {Proceedings of Machine Learning Research},
  volume    = {89},
  pages     = {2760--2769},
  publisher = {PMLR},
  year      = {2019}
}

@incollection{neyman1959optimal,
  author    = {Neyman, Jerzy},
  title     = {Optimal Asymptotic Tests of Composite Statistical Hypotheses},
  booktitle = {Probability and Statistics: The Harald Cram\'{e}r Volume},
  editor    = {Grenander, Ulf},
  publisher = {Almqvist \& Wiksell},
  address   = {Uppsala},
  pages     = {213--234},
  year      = {1959}
}

@article{newey1994asymptotic,
  author  = {Newey, Whitney K.},
  title   = {The Asymptotic Variance of Semiparametric Estimators},
  journal = {Econometrica},
  volume  = {62},
  number  = {6},
  pages   = {1349--1382},
  year    = {1994}
}

@misc{newey2018crossfitting,
  author        = {Newey, Whitney K. and Robins, James R.},
  title         = {Cross-Fitting and Fast Remainder Rates for Semiparametric Estimation},
  year          = {2018},
  eprint        = {1801.09138},
  archivePrefix = {arXiv},
  primaryClass  = {math.ST}
}

@article{chernozhukov2022riesz,
  author  = {Chernozhukov, Victor and Newey, Whitney K. and Singh, Rahul},
  title   = {Debiased Machine Learning of Global and Local Parameters Using Regularized {Riesz} Representers},
  journal = {The Econometrics Journal},
  volume  = {25},
  number  = {3},
  pages   = {576--601},
  year    = {2022}
}

@article{chernozhukov2023simple,
  author  = {Chernozhukov, Victor and Newey, Whitney K. and Singh, Rahul},
  title   = {A Simple and General Debiased Machine Learning Theorem with Finite-Sample Guarantees},
  journal = {Biometrika},
  volume  = {110},
  number  = {1},
  pages   = {257--264},
  year    = {2023}
}

@article{liu2021adaptive,
  author  = {Liu, Lin and Mukherjee, Rajarshi and Robins, James M. and {Tchetgen Tchetgen}, Eric J.},
  title   = {Adaptive Estimation of Nonparametric Functionals},
  journal = {Journal of Machine Learning Research},
  volume  = {22},
  number  = {99},
  pages   = {1--66},
  year    = {2021}
}

@misc{bonvini2024doubly,
  author        = {Bonvini, Matteo and Kennedy, Edward H. and Dukes, Oliver and Balakrishnan, Sivaraman},
  title         = {Doubly-Robust Inference and Optimality in Structure-Agnostic Models with Smoothness},
  year          = {2024},
  eprint        = {2405.08525},
  archivePrefix = {arXiv},
  primaryClass  = {stat.ME}
}

@misc{liu2026inadmissibility,
  author        = {Liu, Lin and Mukherjee, Rajarshi and Robins, James M.},
  title         = {On the Asymptotic Inadmissibility of Double Machine Learning Estimators Under Structure-Agnostic Models},
  year          = {2026},
  eprint        = {2606.22391},
  archivePrefix = {arXiv},
  primaryClass  = {math.ST}
}

@article{dukes2024doublyrobust,
  author  = {Dukes, Oliver and Vansteelandt, Stijn and Whitney, David},
  title   = {On Doubly Robust Inference for Double Machine Learning in Semiparametric Regression},
  journal = {Journal of Machine Learning Research},
  volume  = {25},
  number  = {279},
  pages   = {1--46},
  year    = {2024}
}
\appendix
\clearpage
\section{Omitted Proofs in Section \ref{sec:thms}}\label{app:proofs_thms}
\begin{proof}[Proof of Theorem \ref{thm:lower_bound}]
The proof of Theorem \ref{thm:lower_bound} relies on a sequence of lemmas and the construction of a special model class.

For $x=(x_1,\cdots,x_d)\in[0,1]^d$, we write the canonical binary expansion of $x_1$ as
\[
x_1=\sum_{l=1}^{\infty}B_l(x_1)2^{-l}.
\]
Define $\phi_j(x):=(-1)^{B_j(x_1)}$, $j=1,2,\cdots$.
Since we know that $X$ satisfies the uniform distribution $\nu_d$. We have that $\cbr{B_j(X_1)}_{j=1}^{\infty}$ is a sequence of independent Bernoulli random variables $\text{Bern}(1/2)$. Thus, we have that 
\[
\phi_j\cbr{-1,1},\ \EE_X[\phi_j(X)]=0,\ \EE_X[\phi_j(X)\phi_k(X)]=\II(j=k).
\]
For any $t\ge 0$ and $M\in\ZZ_+$, we define $\cG(t,M)$ to be the code class $\cbr{0,t\phi_1,t\phi_2,\cdots,t\phi_M}$.
\begin{lemma}\label{lem:Rad_com_G(t,M)}
    For every $t\ge 0$, $M\ge 1$, and $\delta>0$,
    \[
    \cR_n(\delta,\partial \cG(t,M))\le 4t\sqrt{\frac{\log(M+1)}{n}}.
    \]
\end{lemma}
Let $\cP_{\text{PLM}}$ denote the collection of all probability laws that satisfy the partial linear model and the other regularization assumptions in Section \ref{sec:setup}, i.e.,
\[
\cP_{\text{PLM}}:=\cbr{\PP_{(X,T,Y)}: \exists \pi\in\cT_3,\mu\in\cT_3,|\beta|\le 3\ s.t.\ X\sim\nu_d, T=\pi(X)+\varepsilon_T,\ Y=\beta T+\mu(X)+\varepsilon_Y}
\]
Equivalently, we have that 
\[
\cP(\cF_{\mu},\cF_{\pi})=\cbr{\PP\in \cP_{\text{PLM}}: \mu\in\cF_{\mu},\pi\in\cF_{\pi}}.
\]
To prove the concrete lower bound, we first need to construct a finite subset of $\cP_{\text{PLM}}$ whose treatment-effect values are hard to distinguish. Thus, for any $M\in\ZZ^+$, we define the following data generating distributions $P_*,P_1,\cdots,P_M$.

First, let $Q$ be the joint distribution of $(X,T,Y)$ such that $X\sim\nu_d$ and $T,Y$ are independent Rademacher random variables that are independent of $X$. We fix $0\le\alpha\le\frac{1}{8}$, $0\le r\le\frac{1}{8}$, $b=\alpha r$.

We now define the Radon-Nikodym derivative between $P_*,P_1,\cdots,P_M$ and $Q$ to determine the data generation distributions.
Specifically, we define 
\[
\frac{dP_*}{dQ}(x,t,y);=1+bty,\ \frac{dP_j}{dQ}(x,t,y):=\{1+\alpha t\phi_j(x)\}\{1+ry\phi_j(x)\}.
\]
By direct algebra, we have that $\EE_Q[1+bTY]=1+0=1$.
Thus, $P_*$ is a joint density. Moreover, by independence of $T,Y$, we have
\begin{align*}
    \EE_Q[1+\alpha T\phi_j(X)|X]=\EE_Q[1+rY\phi_j(X)|X]=1.
\end{align*}
Hence, $P_1,\cdots,P_M$ are joint densities.

Now, we identify their parameters in the partial linear model.
For $P_*$, we have that
\[
\EE_{P^*}[T|X=x]=1P_*(T=1|X=x)+(-1)P_*(T=-1|X=x)=0.
\]
\[
\EE_{P_*}[Y|X=x,T=t]=\sum_{y\in\{-1,1\}}y\frac{1+bty}{2}=bt.
\]
Therefore, under $P_*$, the parameters in the partial linear model is $(\beta,\pi,\mu)_{P_*}=(b,0,0)$.

Similarly, under $P_j$, we have that
\[
P_j(T=t,Y=y|X=x)=\frac{1+\alpha t\phi_j(x)}{2}\frac{1+ry\phi_j(x)}{2},
\]
which can be factorized. Thus, $T$ and $Y$ are conditionally independent given $X$. We have
\[
\EE_{P_j}[T|X=x]=\sum_{t\in\{-1,1\}}t\frac{1+\alpha t\phi_j(x)}{2}=\alpha\phi_j(x),
\]
\[
\EE_{P_j}[Y|X=x,T=t]=\sum_{y\in\{-1,1\}}y \frac{1+ry\phi_j(x)}{2}=r\phi_j(x).
\]
Hence, we know that $(\beta,\pi,\mu)_{P_j}=(0,\alpha\phi_j,r\phi_j)$.

Under $P_*$, we have $\varepsilon_T=T$, $\varepsilon_Y=Y-bT$. Thus, we have that $|\varepsilon_T|\vee|\varepsilon_Y|\le \frac{9}{8}$. $\EE_{P_*}[\varepsilon_T^2|X]=1$, $\EE_{P_*}[\varepsilon_T|X]=0$, $\EE_{P_*}[\varepsilon_Y|X=x,T=t]=bt-bt=0$. Thus, the data generation distribution $P_*$ satisfies all the assumptions in Theorem \ref{thm:lower_bound}, i.e., $P_*\in\cP_{\text{PLM}}$.

Similarly, under $P_j$, $\varepsilon_T=T-\alpha\phi_j(X)$, $\varepsilon_Y=Y-r\phi_j(X)$. We again have that $|\varepsilon_T|\vee|\varepsilon_Y|\le \frac{9}{8}$. $\EE_{P_j}[\varepsilon_Y|X,T]=r\phi_j(X)-r\phi_j(X)=0$, $\EE_{P_j}[\varepsilon_T|X]=0$, $\EE_{P_j}[\varepsilon_T^2|X]=1+\alpha^2-2\alpha^2=1-\alpha^2> 1/3$.
Therefore, we obtain that $P_j\in\cP_{\text{PLM}}$ as well.

Our next lemma shows that these data generation distributions are very hard to distinguish. 
\begin{lemma}\label{lem:chi2_dist}
    Set $D=\frac{\alpha^2+r^2-2b^2}{1-b^2}$. For any $N\in\ZZ^+$, we denote $\overline{P}_N$ as the uniform mixture of the \(M\) \(N\)-sample distributions \(P_j^{\otimes N}\), i.e., $\overline{P}_N=\frac{1}{M}\sum_{j=1}^MP_j^{\otimes N}$. Then we have that
    \[
    \chi^2(\overline{P}_N\|P_*^{\otimes N})=\frac{(1+D)^N-1}{M}.
    \]
    Hence, we have
    \[
    \texttt{D}_\text{TV}(\overline{P}_N\|P_*^{\otimes N})\le \frac{1}{2}\sqrt{\frac{(1+D)^N-1}{M}}.
    \]
\end{lemma}
With Lemma \ref{lem:chi2_dist} and the sample number $N$, we choose $M=\ceil{4e^{2ND}}$ to obtain
\[
 \texttt{D}_\text{TV}(\overline{P}_N\|P_*^{\otimes N})\le \frac{1}{4}.
\]
For any estimator $\hat{\beta}$, define the event $\cA=\cbr{\hat{\beta}\ge \frac{b}{2}}$. On $\cA^c$, we have $(\hat{\beta}-b)^2\ge \frac{b^2}{4}$. Thus, by Markov inequality, we have
\begin{align*}
\EE_{P_*^{\otimes N}}[(\hat{\beta}-b)^2]+\EE_{\overline{P}_N}[\hat{\beta}^2]\ge& \frac{b^2}{4}(P_*^{\otimes N}(\cA^c)+\overline{P}_N(\cA))=\frac{b^2}{4}\cbr{1-P_*^{\otimes} N(\cA)+\overline{P}_N(\cA)}\\
\ge& \frac{b^2}{4}(1-\texttt{D}_\text{TV}(\overline{P}_N\|P_*^{\otimes N}))\ge \frac{3b^2}{16}.
\end{align*}
Thus, we have that
\[
\max\cbr{\EE_{P_*^{\otimes N}}[(\hat{\beta}-b)^2],\max_{1\le j\le M}\EE_{P_j^{\otimes N}}[\hat{\beta}^2]}\ge\frac{3b^2}{32}.
\]
This is the lower bound test engine that we use. Recall that we have $N=2n$ samples.

Our next step is to construct the function classes $\cF_{\pi}$ $\cF_{\mu}$ in the lower minimax bound argument concretely so that they lie in $\cH_n(a_q,s_q)$. To transform our test engine to the minimax risk bound, we first describe the following process.

\paragraph{From Test to Minimax Risk.} Assuming that for some constructed $\cF_{\mu}$, $\cF_{\pi}$, we have proved that $\cF_{\pi}\in\cH_n(a_\pi,s_\pi)$, $\cF_\mu\in\cH_{n}(a\mu,s_\mu)$ and that
\[
\cbr{P_*,P_1,\cdots,P_M}\subset \cP(\cF_\mu,\cF_\pi).
\]
Then, by the definition of minimax risk, with $N$ observed samples, we have that
\begin{align*}
     \cE(n,\{a_q,s_q\}_{q\in\{\mu,\pi\}})=&\sup_{\tilde{\cF}_{q}\in\cH_n(a_q,s_q),q\in\{\mu,\pi\}}\inf_{\hat{\beta}}\sup_{\PP\in\cP(\tilde{\cF}_{\mu},\tilde{\cF}_\pi)}\EE_{\PP^{\otimes N}}[|\hat{\beta}-\beta(\PP)|^2]\\
    \ge&\inf_{\hat{\beta}}\sup_{\PP\in\cP(\cF_{\mu},\cF_\pi)}\EE_{\PP^{\otimes N}}[|\hat{\beta}-\beta(\PP)|^2]\\
    \ge& \inf_{\hat{\beta}}\max\cbr{\EE_{P_*^{\otimes N}}[(\hat{\beta}-b)^2],\max_{1\le j\le M}\EE_{P_j^{\otimes N}}[\hat{\beta}^2]}\\
    \ge&\frac{3b^2}{32}
\end{align*}
Define $C_*=16+\log 6$ and $\kappa$ small enough such that $0<\kappa\le \min\cbr{\frac{1}{24},(4\sqrt{C_*})^{-1/2}}$.

Note that, roughly speaking, we claim
\[
\sqrt{ \cE(n,\{a_q,s_q\}_{q\in\{\mu,\pi\}})}\gtrsim  \frac{1}{n^{1/2}}+a_\mu a_\pi+\min\{a_{\pi}s_\mu+s_\pi^2,a_\mu s_\pi+s_\mu^2\}.
\]
Thus, we consider the following four cases.
\paragraph{Case I.} In this case, we define $\alpha:=\kappa a_{\pi}$, $r:=\kappa a_\mu$, $b=\alpha r=\kappa^2(a_\mu a_\pi)$. Recall that we denote $D=\frac{\alpha^2+r^2-2b^2}{1-b^2}$ and $M=\ceil{4e^{2ND}}$.
We choose $\cF_{\pi}=\cbr{0,\alpha\phi_1,\cdots,\alpha\phi_M}$, $\cG_{\pi}=\cbr{0}$, $\cF_{\mu}=\cbr{0,r\phi_1,\cdots,r\phi_M}$, $G_{\mu}=\cbr{0}$.

 By the parameter function settings of $P_*,P_1,\cdots,P_M$, we have that $\cbr{P_*,P_1,\cdots,P_M}\subset \cP(\cF_{\mu},\cF_\pi)$. Then, by direct algebra, we have that
\begin{align*}
    \sup_{f\in\cF_{\pi}}\inf_{g\in\cG_\pi}\|f-g\|_2=\sup_{f\in\cF_\pi}\|f\|_2=\alpha=\kappa a_{\pi}\le a_{\pi}.
\end{align*}
Similarly, we have $\sup_{f\in\cF_{\mu}}\inf_{g\in\cG_\mu}\|f-g\|_2=ra_\mu\le a_\mu$.
Since $\cG_{\pi}$ and $\cG_{\mu}$ are singletons, we have $\cR_n(\delta,\partial \cG_{\pi})=\cR_n(\delta,\partial \cG_{\mu})=0$. Thus, we have
\[
\cF_{\pi}\in\cH_n(a_\pi,s_{\pi}),\ \cF_\mu\in\cH_n(a_\mu,s_\mu).
\]
Plugging this into the test engine machine, we have that $$ \cE(n,\{a_q,s_q\}_{q\in\{\mu,\pi\}})\ge \frac{3\kappa^4 (a_\mu a_\pi)^2}{32}.$$

\paragraph{Case II.} We define $\alpha:=\kappa (a_\pi\wedge s_\mu)$, $r=\kappa s_\mu$, and $b=\alpha r=\kappa^2s_\mu(a_\pi\wedge s_\mu)$. We will use this construction in the risk lower bound under the condition that $b=\kappa^2s_\mu(a_\pi\wedge s_\mu)\ge \frac{1}{\sqrt{n}}$. Recall that we denote $D=\frac{\alpha^2+r^2-2b^2}{1-b^2}$ and $M=\ceil{4e^{2ND}}$.
We choose $\cF_{\pi}=\cbr{0,\alpha\phi_1,\cdots,\alpha\phi_M}$, $\cG_{\pi}=\cbr{0}$, $\cF_{\mu}=\cbr{0,r\phi_1,\cdots,r\phi_M}$, $G_{\mu}=\cbr{0,r\phi_1,\cdots,r\phi_M}$.

Again, we have that $\cbr{P_*,P_1,\cdots,P_M}\subset \cP(\cF_{\mu},\cF_\pi)$. Repeating the process above, we have
\begin{align*}
    \sup_{f\in\cF_{\pi}}\inf_{g\in\cG_\pi}\|f-g\|_2=\sup_{f\in\cF_\pi}\|f\|_2=\alpha=\kappa(a_\pi\wedge s_\mu)\le a_\pi. 
\end{align*}
$\cR_n(\delta,\partial \cG_\pi)=0$ Thus, we obtain that $\cF_\pi\in\cH_n(a_\pi,s_\pi)$. Secondly, $\sup_{f\in\cF_{\pi}}\inf_{g\in\cG_\pi}\|f-g\|_2=0$. Now it remains to bound $\cR_n(\delta,\partial\cG_\mu)$.

First, by definition, we have that $D\le2(\alpha^2+r^2)\le 4\kappa^2s_\mu^2$ and that $b=\kappa^2s_\mu(a_\pi\wedge s_\mu)\le \kappa^2s_\mu^2$. Now, recall the condition that $b\ge\frac{1}{\sqrt{n}}$, we obtain $\kappa^2ns_\mu^2\ge\sqrt{n}\ge 1$ by combining these inequalities together. Therefore, recall the value of $C_*$, we have that
\begin{align*}
    \log(M+1)\le \log 6+2ND\le \log 6+4nD\le \log 6+16\kappa^2ns_{\mu}^2\le (16+\log 6)\kappa^2ns_\mu^2=C_*\kappa^2ns_\mu^2.
\end{align*}
Finally, by Lemma \ref{lem:Rad_com_G(t,M)}, for every $\delta\ge s_\mu$, we have that
\begin{align*}
    \cR_n(\delta,\partial\cG_\mu)\le 4r\sqrt{\frac{\log (M+1)}{n}}\le 4\kappa s_\mu\sqrt{C_*\kappa^2s_\mu^2}=4\sqrt{C_*}\kappa^2s_\mu^2\le s_\mu^2\le \delta s_\mu.
\end{align*}
Thus, we have that $\cF_\mu\in \cH_n(a_\mu,s_\mu)$. We can apply the test engine now to obtain that when $\kappa^2s_\mu(a_\pi\wedge s_\mu)\ge\frac{1}{\sqrt{n}}$,
\[
\cE(n,\{a_q,s_q\}_{q\in\{\mu,\pi\}})\ge \frac{3\kappa^4s_\mu^2(a_\pi\wedge s_\mu)^2}{32}.
\]
\paragraph{Case III.} This Case is symmetric to Case II above.  We define $\alpha=\kappa s_\pi$, $r=\kappa(a_\mu\wedge s_\pi)$, $b=\alpha r=\kappa^2s_\pi(a_\mu\wedge s_\pi)$. We choose $\cF_\pi=\cbr{0,\alpha\phi_1,\cdots,\alpha\phi_M}=\cG_\pi$, $\cF_{\mu}=\cbr{0,r\phi_1,\cdots,r\phi_M}$, and $\cG_\mu=\cbr{0}$.

Then, by the same argument above, we conclude that when $\kappa^2s_\pi(a_\mu\wedge s_\pi)\ge 1/\sqrt{n}$, $\cF_\pi\in\cH_n(a_\pi,s_\pi)$, $\cF_\mu\in\cH_n(a_\mu,s_\mu)$. Therefore, when $\kappa^2s_\pi(a_\mu\wedge s_\pi)\ge 1/\sqrt{n}$, 
\[
\cE(n,\{a_q,s_q\}_{q\in\{\mu,\pi\}})\ge \frac{3\kappa^4s_\pi^2(a_\mu\wedge s_\pi)^2}{32}.
\]
\paragraph{Case IV.} In this case, we choose $\alpha=\kappa(s_\mu\wedge s_\pi)$, $r=\kappa (s_\mu\wedge s_\pi)$, $b=\alpha r=\kappa^2(s_\mu\wedge s_\pi)^2$. Recall that $D=\frac{\alpha^2+r^2-2b^2}{1-b^2}$ and $M=\ceil{4e^{2ND}}$. In this case, we choose $\cF_\pi=\cG_\pi=\cbr{0,\alpha\phi_1,\cdots,\alpha\phi_M}$, $\cF_\mu=\cG_\mu=\cbr{0,r\phi_1,\cdots,r\phi_M}$. 

Obviously, we have that $\sup_{f\in\cF_\pi}\inf_{g\in\cG_\pi}\|f-g\|_2=0$, $\sup_{f\in\cF_\mu}\inf_{g\in\cG_\mu}\|f-g\|_2=0$. We now claim that when $\kappa^2(s_\mu\wedge s_\pi)^2\ge\frac{1}{\sqrt{n}}$, $$\cR_n(\delta_1;\partial\cG_\pi)\le \delta_1 s_\pi, \cR_n(\delta_2;\partial\cG_\mu)\le \delta_2 s_\mu,\ \text{for}\ \delta_1\ge s_\pi,\delta_2\ge s_\mu.$$ 

To prove the claim, notice that $D\le 2(\alpha^2+r^2)\le 4\kappa^2(s_\mu\wedge s_\pi)^2$. By the condition $\kappa^2(s_\mu\wedge s_\pi)^2\ge\frac{1}{\sqrt{n}}$, we obtain that $\kappa^2n(s_\mu\wedge s_\pi)^2\ge\sqrt{n}\ge 1$. Thus,
\[
\log(M+1)=\log 6+2ND= \log 6+4nD\le \log 6+16\kappa^2n(s_{\mu}\wedge s_\pi)^2\le (16+\log 6)\kappa^2n(s_{\mu}\wedge s_\pi)^2.
\]
Thus, recall that we set $C_*=16+\log 6$, for $\delta\ge s_\pi$, we have that
\begin{align*}
    \cR_n(\delta;\partial\cG_\pi)\le 4\alpha\sqrt{\frac{\log(M+1)}{n}}\le 4\sqrt{C_*}\kappa^2(s_\mu\wedge s_\pi)^2\le (s_{\mu}\wedge s_\pi)^2\le s_\pi^2\le\delta s_\pi.
\end{align*}
Similarly, for $\delta\ge s_\mu$, we have
\[
\cR_n(\delta;\partial\cG_\mu)\le \delta s_\mu.
\]
Thus, we prove the claim and therefore obtain
\[
\cF_\pi\in\cH_n(a_\pi,s_\pi),\ \cF_\mu\in\cH_n(a_\mu,s_\mu).
\]
Applying the test engine, we have that 
\[
\cE(n,\{a_q,s_q\}_{q\in\{\mu,\pi\}})\ge \frac{3\kappa^4(s_\mu\wedge s_\pi)^4}{32}.
\]
Finally, by \citet{balakrishnan2026fundamental}, we know that $\cE(n,\{a_q,s_q\}_{q\in\{\mu,\pi\}})\ge \frac{C}{n}$ for some universal constant $C>0$. Combining these circumstances together, we obtain that
\[
\cE(n,\{a_q,s_q\}_{q\in\{\mu,\pi\}})\ge c\max\cbr{\frac{1}{n},(a_\mu a_\pi)^2,s_\mu^2(a_\pi\wedge s_\mu)^2,s_\pi^2(a_\mu\wedge s_\pi)^2,(s_\mu\wedge s_\pi)^4}.
\]
We have the following lemma.
\begin{lemma}[Two-sided truncation inequality]
\label{lem:two-sided-truncation}
Let \(p,u,v,\rho,\tau\ge 0\) satisfy $uv\le p\sqrt{\rho\tau}$.
Define
\[
x:=u\wedge\rho,
\ 
y:=v\wedge\tau,
\ 
z:=\rho\wedge\tau.
\]
Then, we have
\[
\min\{u+\tau,v+\rho\}
\le p+x+y+z.
\]
Consequently,
\[
p+\min\{u+\tau,v+\rho\}
\le 2p+x+y+z
\le 5\max\{p,x,y,z\}.
\]
\end{lemma}
In our lower bound proof, we take $p=a_\mu a_\pi$, $u=a_\pi s_\mu$, $v=a_\mu s_\pi$, $\rho=s_\mu^2$, $\tau=s_\pi^2$. Then, we have that
\[
B=p+\min\cbr{u+\tau,v+\rho}\le 5\max\cbr{p,x,y,z}.
\]
Thus, we have that
\[
\frac{1}{n}+B^2\le 26\max\cbr{\frac{1}{n},(a_\mu a_\pi)^2,s_\mu^2(a_\pi\wedge s_\mu)^2,s_\pi^2(a_\mu\wedge s_\pi)^2,(s_\mu\wedge s_\pi)^4}.
\]
Finally, plugging this back, we get
\[
\cE(n,\{a_q,s_q\}_{q\in\{\mu,\pi\}})\ge \frac{c}{26}\sbr{1\wedge\cbr{\frac{1}{n}+\rbr{a_{\mu}a_\pi+\min\{a_{\pi}s_\mu+s_\pi^2,a_\mu s_\pi+s_\mu^2\}}^2}}.
\]
By the AM-GM inequality $\sqrt{a^2+b^2}\ge\frac{a+b}{\sqrt{2}}$, we have 
\[
\sqrt{\cE(n,\{a_q,s_q\}_{q\in\{\mu,\pi\}})}\gtrsim 1\wedge\cbr{\frac{1}{\sqrt{n}}+\rbr{a_{\mu}a_\pi+\min\{a_{\pi}s_\mu+s_\pi^2,a_\mu s_\pi+s_\mu^2\}}}.
\]
Thus, we finish the proof.
\end{proof}

\begin{proof}[Proof of Corollary \ref{cor:final_characterization}]
We first compare the deterministic remainder appearing in the upper
bound with the quantity appearing in the lower bound.

Recall that $\Gamma=
\min\left\{
a_\mu(a_\pi+s_\pi)+(a_\mu+s_\mu)^2,\,
a_\pi(a_\mu+s_\mu)+(a_\pi+s_\pi)^2
\right\}$, and
\[
B
=
a_\mu a_\pi+
\min\left\{
a_\pi s_\mu+s_\pi^2,\,
a_\mu s_\pi+s_\mu^2
\right\}.
\]
We first prove
\[
\Gamma
\le
4\left(
a_\mu a_\pi+a_\pi s_\mu+s_\pi^2
\right).
\]
There are two cases.

\noindent
\emph{Case 1:}
If $
a_\mu+s_\mu\le a_\pi+s_\pi$, we first show that
\[
s_\pi(a_\mu+s_\mu)
\le
a_\pi(a_\mu+s_\mu)+s_\pi^2.
\]

If \(s_\pi\le a_\pi\), then we have
$s_\pi(a_\mu+s_\mu)\le
a_\pi(a_\mu+s_\mu)
\le
a_\pi(a_\mu+s_\mu)+s_\pi^2$.

If \(s_\pi>a_\pi\), then we have
$(s_\pi-a_\pi)(a_\mu+s_\mu)
\le
(s_\pi-a_\pi)(a_\pi+s_\pi)$.

Recall that $a_\mu+s_\mu\le a_\pi+s_\pi$. Moreover, by algebra, we have $(s_\pi-a_\pi)(a_\pi+s_\pi)
=
s_\pi^2-a_\pi^2
\le
s_\pi^2$.

Therefore, we have $(s_\pi-a_\pi)(a_\mu+s_\mu)\le s_\pi^2$,
which is equivalent to
\[
s_\pi(a_\mu+s_\mu)
\le
a_\pi(a_\mu+s_\mu)+s_\pi^2.
\]

Thus, in either case, we have
\[
s_\pi(a_\mu+s_\mu)
\le
a_\pi(a_\mu+s_\mu)+s_\pi^2.
\]
Consequently, we have
\begin{align*}
(a_\mu+s_\mu)(a_\pi+s_\pi)=
a_\pi(a_\mu+s_\mu)
+
s_\pi(a_\mu+s_\mu)
\le
2\left\{
a_\pi(a_\mu+s_\mu)+s_\pi^2
\right\}
=
2\left\{
a_\mu a_\pi+a_\pi s_\mu+s_\pi^2
\right\}.
\end{align*}

Since $a_\mu\le a_\mu+s_\mu$
and $a_\mu+s_\mu\le a_\pi+s_\pi$, we have
\[
a_\mu(a_\pi+s_\pi)
\le
(a_\mu+s_\mu)(a_\pi+s_\pi)
\]
and
\[
(a_\mu+s_\mu)^2
\le
(a_\mu+s_\mu)(a_\pi+s_\pi).
\]
It follows that
\begin{align*}
a_\mu(a_\pi+s_\pi)+(a_\mu+s_\mu)^2\le
2(a_\mu+s_\mu)(a_\pi+s_\pi)
\le
4\left\{
a_\mu a_\pi+a_\pi s_\mu+s_\pi^2
\right\}.
\end{align*}
Because \(\Gamma\) is the minimum of the two upper-bounds, we have that $\Gamma
\le
a_\mu(a_\pi+s_\pi)+(a_\mu+s_\mu)^2$, and hence
\[
\Gamma
\le
4\left\{
a_\mu a_\pi+a_\pi s_\mu+s_\pi^2
\right\}.
\]

\medskip

\noindent
\emph{Case 2:} If
$a_\mu+s_\mu>a_\pi+s_\pi$, because $a_\mu+s_\mu>a_\pi+s_\pi\ge a_\pi$, we have
\[
a_\pi(a_\mu+s_\mu)\ge a_\pi^2.
\]
Therefore, we obtain
\begin{align*}
a_\mu a_\pi+a_\pi s_\mu+s_\pi^2=
a_\pi(a_\mu+s_\mu)+s_\pi^2\ge
a_\pi^2+s_\pi^2.
\end{align*}
Using the fact that $(a_\pi+s_\pi)^2
\le
2a_\pi^2+2s_\pi^2$, we obtain
\[
(a_\pi+s_\pi)^2
\le
2\left\{
a_\mu a_\pi+a_\pi s_\mu+s_\pi^2
\right\}.
\]
Also,
\[
a_\pi(a_\mu+s_\mu)
\le
a_\mu a_\pi+a_\pi s_\mu+s_\pi^2.
\]
Consequently, we have
\begin{align*}
a_\pi(a_\mu+s_\mu)+(a_\pi+s_\pi)^2
\le
3\left\{
a_\mu a_\pi+a_\pi s_\mu+s_\pi^2
\right\}
\le
4\left\{
a_\mu a_\pi+a_\pi s_\mu+s_\pi^2
\right\}.
\end{align*}
Since $\Gamma$ is the minimum of the two upper bounds, we have $\Gamma
\le
a_\pi(a_\mu+s_\mu)+(a_\pi+s_\pi)^2$, thus, we again conclude that
\[
\Gamma
\le
4\left\{
a_\mu a_\pi+a_\pi s_\mu+s_\pi^2
\right\}.
\]
Therefore, in both cases, we obtain
\[
\Gamma
\le
4\left\{
a_\mu a_\pi+a_\pi s_\mu+s_\pi^2
\right\}.
\]

Following the same logic, our next goal is to prove
\[
\Gamma
\le
4\left\{
a_\mu a_\pi+a_\mu s_\pi+s_\mu^2
\right\}.
\]

Again, there are two cases.

\emph{Case 1:} If $a_\pi+s_\pi\le a_\mu+s_\mu$, we first show that
\[
s_\mu(a_\pi+s_\pi)
\le
a_\mu(a_\pi+s_\pi)+s_\mu^2.
\]

If \(s_\mu\le a_\mu\), then $s_\mu(a_\pi+s_\pi)
\le
a_\mu(a_\pi+s_\pi)
\le
a_\mu(a_\pi+s_\pi)+s_\mu^2$.

If \(s_\mu>a_\mu\), then we have
$(s_\mu-a_\mu)(a_\pi+s_\pi)
\le
(s_\mu-a_\mu)(a_\mu+s_\mu)
=
s_\mu^2-a_\mu^2
\le
s_\mu^2$.

Therefore, either way, we always have
\[
s_\mu(a_\pi+s_\pi)
\le
a_\mu(a_\pi+s_\pi)+s_\mu^2.
\]
It follows that
\begin{align*}
(a_\mu+s_\mu)(a_\pi+s_\pi)=
a_\mu(a_\pi+s_\pi)
+
s_\mu(a_\pi+s_\pi)
\le
2\left\{
a_\mu(a_\pi+s_\pi)+s_\mu^2
\right\}
=
2\left\{
a_\mu a_\pi+a_\mu s_\pi+s_\mu^2
\right\}.
\end{align*}

Since  $
a_\pi\le a_\pi+s_\pi$ and $a_\pi+s_\pi\le a_\mu+s_\mu$, we have
\[
a_\pi(a_\mu+s_\mu)
\le
(a_\pi+s_\pi)(a_\mu+s_\mu)
\]
and
\[
(a_\pi+s_\pi)^2
\le
(a_\pi+s_\pi)(a_\mu+s_\mu).
\]
Therefore, adding the inequalities above together, we have
\begin{align*}
a_\pi(a_\mu+s_\mu)+(a_\pi+s_\pi)^2\le
2(a_\mu+s_\mu)(a_\pi+s_\pi)\le
4\left\{
a_\mu a_\pi+a_\mu s_\pi+s_\mu^2
\right\}.
\end{align*}
Since $\Gamma$ is the minimum of the two upper bounds, we have $\Gamma
\le
a_\pi(a_\mu+s_\mu)+(a_\pi+s_\pi)^2$. Therefore,
we obtain
\[
\Gamma
\le
4\left\{
a_\mu a_\pi+a_\mu s_\pi+s_\mu^2
\right\}.
\]

\medskip

\noindent
\emph{Case 2:} If $a_\pi+s_\pi>a_\mu+s_\mu$, because $a_\pi+s_\pi>a_\mu+s_\mu\ge a_\mu$, we have
\[
a_\mu(a_\pi+s_\pi)\ge a_\mu^2.
\]
It follows that
\begin{align*}
a_\mu a_\pi+a_\mu s_\pi+s_\mu^2=a_\mu(a_\pi+s_\pi)+s_\mu^2\ge
a_\mu^2+s_\mu^2.
\end{align*}
Hence
\[
(a_\mu+s_\mu)^2
\le
2a_\mu^2+2s_\mu^2
\le
2\left\{
a_\mu a_\pi+a_\mu s_\pi+s_\mu^2
\right\}.
\]
Moreover, by direct algebra,  we have
\[
a_\mu(a_\pi+s_\pi)
\le
a_\mu a_\pi+a_\mu s_\pi+s_\mu^2.
\]
Therefore, we get
\begin{align*}
a_\mu(a_\pi+s_\pi)+(a_\mu+s_\mu)^2
\le
3\left\{
a_\mu a_\pi+a_\mu s_\pi+s_\mu^2
\right\}
\le
4\left\{
a_\mu a_\pi+a_\mu s_\pi+s_\mu^2
\right\}.
\end{align*}
Since $\Gamma
\le
a_\mu(a_\pi+s_\pi)+(a_\mu+s_\mu)^2$ by definition,
we again obtain
\[
\Gamma
\le
4\left\{
a_\mu a_\pi+a_\mu s_\pi+s_\mu^2
\right\}.
\]

Therefore, combining the two bounds gives
\begin{align*}
\Gamma
\le
4\min\left\{
a_\mu a_\pi+a_\pi s_\mu+s_\pi^2,\,
a_\mu a_\pi+a_\mu s_\pi+s_\mu^2
\right\}
=
4\left[
a_\mu a_\pi+
\min\left\{
a_\pi s_\mu+s_\pi^2,\,
a_\mu s_\pi+s_\mu^2
\right\}
\right]=
4B.
\end{align*}
Applying the inequality proved in Step 1 with
\(s_\mu,s_\pi\) replaced by \(\bar s_\mu,\bar s_\pi\), we have
\begin{align}
\bar\Gamma
\le
4\Big[
a_\mu a_\pi+
\min\big\{
a_\pi\bar s_\mu+\bar s_\pi^2,\,
a_\mu\bar s_\pi+\bar s_\mu^2
\big\}
\Big].
\label{eq:bargamma-first-bound}
\end{align}

We now prove that
\begin{align}
&a_\mu a_\pi+
\min\left\{
a_\pi\bar s_\mu+\bar s_\pi^2,\,
a_\mu\bar s_\pi+\bar s_\mu^2
\right\}
\le
2\left(B+\rho_n^2\right).
\label{eq:floor-inflation}
\end{align}

There are four cases.

\medskip

\noindent
\emph{Case 1:} If \(s_\mu\ge\rho_n\) and \(s_\pi\ge\rho_n\), then, $\bar s_\mu=s_\mu,
\ 
\bar s_\pi=s_\pi$.

Therefore, the left-hand side of
\eqref{eq:floor-inflation} is exactly \(B\), and hence
\[
B\le2(B+\rho_n^2).
\]

\medskip

\noindent
\emph{Case 2:} If \(s_\mu<\rho_n\) and \(s_\pi<\rho_n\), then $\bar s_\mu=\bar s_\pi=\rho_n$, and therefore we have
\begin{align*}
&a_\mu a_\pi+
\min\left\{
a_\pi\bar s_\mu+\bar s_\pi^2,\,
a_\mu\bar s_\pi+\bar s_\mu^2
\right\}
=
a_\mu a_\pi+\rho_n^2
+\rho_n(a_\mu\wedge a_\pi).
\end{align*}
Since $(a_\mu\wedge a_\pi)^2\le a_\mu a_\pi$, we have
\begin{align*}
2\rho_n(a_\mu\wedge a_\pi)\le
\rho_n^2+(a_\mu\wedge a_\pi)^2
\le
\rho_n^2+a_\mu a_\pi.
\end{align*}
The first inequality is by AM-GM inequality. Equivalently, we have $\rho_n(a_\mu\wedge a_\pi)
\le
\frac12\left(\rho_n^2+a_\mu a_\pi\right)$.

Consequently, we have
\begin{align*}
a_\mu a_\pi+\rho_n^2+\rho_n(a_\mu\wedge a_\pi)
\le
\frac32\left(a_\mu a_\pi+\rho_n^2\right)
\le
2\left(B+\rho_n^2\right),
\end{align*}
where the last inequality uses \(B\ge a_\mu a_\pi\).

\medskip

\noindent
\emph{Case 3:} If \(s_\mu<\rho_n\le s_\pi\), then we know that $\bar s_\mu=\rho_n,
\ 
\bar s_\pi=s_\pi$.
Hence, we have
\begin{align*}
&a_\mu a_\pi+
\min\left\{
a_\pi\bar s_\mu+\bar s_\pi^2,\,
a_\mu\bar s_\pi+\bar s_\mu^2
\right\}
=
a_\mu a_\pi+
\min\left\{
a_\pi\rho_n+s_\pi^2,\,
a_\mu s_\pi+\rho_n^2
\right\}.
\end{align*}
Recall that
\[
(B+\rho_n^2)=(
a_\mu a_\pi+
\min\left\{
a_\pi s_\mu+s_\pi^2,\,
a_\mu s_\pi+s_\mu^2
\right\}+\rho_n^2)
\]
Suppose first that $a_\mu s_\pi+s_\mu^2
\le
a_\pi s_\mu+s_\pi^2$. Then we know that $B=a_\mu a_\pi+a_\mu s_\pi+s_\mu^2$.

Therefore, we have
\begin{align*}
a_\mu a_\pi+
\min\left\{
a_\pi\rho_n+s_\pi^2,\,
a_\mu s_\pi+\rho_n^2
\right\}
\le
a_\mu a_\pi+a_\mu s_\pi+\rho_n^2=
B+\rho_n^2-s_\mu^2
\le
B+\rho_n^2.
\end{align*}

Suppose next that $a_\pi s_\mu+s_\pi^2
<
a_\mu s_\pi+s_\mu^2$. Then we know that $B=a_\mu a_\pi+a_\pi s_\mu+s_\pi^2$.

Consequently, we have
\begin{align*}
a_\mu a_\pi+
\min\left\{
a_\pi\rho_n+s_\pi^2,\,
a_\mu s_\pi+\rho_n^2
\right\}
\le
a_\mu a_\pi+a_\pi\rho_n+s_\pi^2
=
B+a_\pi(\rho_n-s_\mu).
\end{align*}

i) If \(a_\pi\le s_\mu\), then we have $a_\pi(\rho_n-s_\mu)
\le
s_\mu(\rho_n-s_\mu)
\le
\rho_n^2$.

ii) If \(a_\pi>s_\mu\), then we have that $a_\pi s_\mu+s_\pi^2
<
a_\mu s_\pi+s_\mu^2$, which implies that
\begin{align*}
a_\mu s_\pi
>
s_\pi^2+s_\mu(a_\pi-s_\mu)
>
s_\pi^2.
\end{align*}
Therefore, we have $a_\mu>s_\pi\ge\rho_n$, and hence
\[
a_\pi(\rho_n-s_\mu)
\le
a_\pi\rho_n
\le
a_\mu a_\pi.
\]
Thus, in either subcase, we have
\[
a_\pi(\rho_n-s_\mu)
\le
a_\mu a_\pi+\rho_n^2.
\]
It follows that
\begin{align*}
&a_\mu a_\pi+
\min\left\{
a_\pi\rho_n+s_\pi^2,\,
a_\mu s_\pi+\rho_n^2
\right\}
\le
B+a_\mu a_\pi+\rho_n^2
\le
2B+\rho_n^2\le
2(B+\rho_n^2).
\end{align*}

\medskip

\noindent
\emph{Case 4:} \(s_\pi<\rho_n\le s_\mu\). In this case, we have $\bar s_\pi=\rho_n,
\ 
\bar s_\mu=s_\mu$, and hence
\begin{align*}
a_\mu a_\pi+
\min\left\{
a_\pi\bar s_\mu+\bar s_\pi^2,\,
a_\mu\bar s_\pi+\bar s_\mu^2
\right\}
=
a_\mu a_\pi+
\min\left\{
a_\pi s_\mu+\rho_n^2,\,
a_\mu\rho_n+s_\mu^2
\right\}.
\end{align*}

Suppose first that
$a_\pi s_\mu+s_\pi^2
\le
a_\mu s_\pi+s_\mu^2$. Then by definition, we have
\[
B=a_\mu a_\pi+a_\pi s_\mu+s_\pi^2,
\]
and therefore
\begin{align*}
&a_\mu a_\pi+
\min\left\{
a_\pi s_\mu+\rho_n^2,\,
a_\mu\rho_n+s_\mu^2
\right\}\le
a_\mu a_\pi+a_\pi s_\mu+\rho_n^2
=
B+\rho_n^2-s_\pi^2
\le
B+\rho_n^2.
\end{align*}

Suppose next that $a_\mu s_\pi+s_\mu^2
<
a_\pi s_\mu+s_\pi^2$. Then, we know that
\[
B=a_\mu a_\pi+a_\mu s_\pi+s_\mu^2.
\]
Consequently, we have
\begin{align*}
&a_\mu a_\pi+
\min\left\{
a_\pi s_\mu+\rho_n^2,\,
a_\mu\rho_n+s_\mu^2
\right\}
\le
a_\mu a_\pi+a_\mu\rho_n+s_\mu^2=
B+a_\mu(\rho_n-s_\pi).
\end{align*}

i) If \(a_\mu\le s_\pi\), then we obtain $a_\mu(\rho_n-s_\pi)
\le
s_\pi(\rho_n-s_\pi)
\le
\rho_n^2$.

ii) If \(a_\mu>s_\pi\), then we have $a_\mu s_\pi+s_\mu^2
<
a_\pi s_\mu+s_\pi^2$, which implies that
\begin{align*}
a_\pi s_\mu
>
s_\mu^2+s_\pi(a_\mu-s_\pi)
>
s_\mu^2.
\end{align*}
Therefore, we get that $a_\pi>s_\mu\ge\rho_n$,
and hence obtain $a_\mu(\rho_n-s_\pi)
\le
a_\mu\rho_n
\le
a_\mu a_\pi$.

Therefore, in either subcase, we have that
\[
a_\mu(\rho_n-s_\pi)
\le
a_\mu a_\pi+\rho_n^2.
\]
It follows that
\begin{align*}
a_\mu a_\pi+
\min\left\{
a_\pi s_\mu+\rho_n^2,\,
a_\mu\rho_n+s_\mu^2
\right\}
\le
B+a_\mu a_\pi+\rho_n^2
\le
2B+\rho_n^2
\le
2(B+\rho_n^2).
\end{align*}

These four cases prove \eqref{eq:floor-inflation}. Combining
\eqref{eq:bargamma-first-bound} and \eqref{eq:floor-inflation} together, we have
\[
\bar\Gamma
\le
8(B+\rho_n^2).
\]

Our next task if to convert the tail bound in Theorem \ref{thm:upper_bound} into squared risk.

Fix $ \mathcal F_\mu\in\mathcal H_n(a_\mu,s_\mu),\ \mathcal F_\pi\in\mathcal H_n(a_\pi,s_\pi)$ and let \(\mathcal G_\mu,\mathcal G_\pi\) be the corresponding
learner classes. Let \(\widehat\beta\) be the  estimator in
Theorem~\ref{thm:upper_bound}.

Theorem~\ref{thm:upper_bound} implies that, for every $\mathbb P\in\mathcal P(\mathcal F_\mu,\mathcal F_\pi)$ and every \(t\in[1,cn]\), we have
\begin{align}
&\mathbb P^{2n}\left(
\left|
\widehat\beta-\beta(\mathbb P)
\right|
>
C_0\bar\Gamma+
C_0\left\{
\sqrt{\frac tn}+\frac tn
\right\}
\right)\le
C_1\left(e^{-t}+\frac1n\right),
\label{eq:tail-for-integration}
\end{align}
where \(C_0,C_1,c>0\) are uniform constants.

Equivalently, we have
\begin{align}
&\mathbb P^{2n}\left(
\left(
\left|
\widehat\beta-\beta(\mathbb P)
\right|
-C_0\bar\Gamma
\right)_+
>
C_0\left\{
\sqrt{\frac tn}+\frac tn
\right\}
\right)\le
C_1\left(e^{-t}+\frac1n\right).
\label{eq:positive-part-tail}
\end{align}

Recall that$
\widehat\beta\in[-3,3]
\ \text{and}\ 
\beta(\mathbb P)\in[-3,3]$, we have $\left|
\widehat\beta-\beta(\mathbb P)
\right|
\le6$.

Thus, we have
\[
\left(
\left|
\widehat\beta-\beta(\mathbb P)
\right|
-C_0\bar\Gamma
\right)_+
\le6.
\]

Applying the layer-cake identity, we have
\begin{align*}
&\mathbb E_{\mathbb P^{2n}}
\left[
\left(
\left|
\widehat\beta-\beta(\mathbb P)
\right|
-C_0\bar\Gamma
\right)_+^2
\right]=
\int_0^\infty
2u\,
\mathbb P^{2n}\left(
\left(
\left|
\widehat\beta-\beta(\mathbb P)
\right|
-C_0\bar\Gamma
\right)_+
>u
\right)\,du.
\end{align*}

For $0\le u\le
C_0\left(\frac1{\sqrt n}+\frac1n\right)$, we use the trivial probability bound by one. This part of the integral is at most
\[
C_0^2\left(\frac1{\sqrt n}+\frac1n\right)^2
\lesssim
\frac1n.
\]

For the remaining part, by the change of variables $u
=
C_0\left\{
\sqrt{\frac tn}+\frac tn
\right\}$, we have
\begin{align*}
u^2
&=
C_0^2
\left(
\frac tn+
\frac{2t^{3/2}}{n^{3/2}}+
\frac{t^2}{n^2}
\right),
\end{align*}
and hence
\[
\frac{d}{dt}u^2
=
C_0^2
\left(
\frac1n+
\frac{3\sqrt t}{n^{3/2}}+
\frac{2t}{n^2}
\right).
\]

Using \eqref{eq:positive-part-tail}, we obtain
\begin{align*}
\mathbb E_{\mathbb P^{2n}}
\left[
\left(
\left|
\widehat\beta-\beta(\mathbb P)
\right|
-C_0\bar\Gamma
\right)_+^2
\right]
&\le
C_0^2\left(\frac1{\sqrt n}+\frac1n\right)^2+
C_0^2C_1
\int_1^{cn}
\left(
\frac1n+
\frac{3\sqrt t}{n^{3/2}}+
\frac{2t}{n^2}
\right)
\left(e^{-t}+\frac1n\right)\,dt
\\
&+
36C_1\left(e^{-cn}+\frac1n\right).
\end{align*}

The exponentially weighted part satisfies
\begin{align*}
&\int_1^{cn}
\left(
\frac1n+
\frac{3\sqrt t}{n^{3/2}}+
\frac{2t}{n^2}
\right)e^{-t}\,dt
\le
\frac1n\int_1^\infty e^{-t}\,dt
+
\frac3{n^{3/2}}\int_1^\infty\sqrt t\,e^{-t}\,dt
+
\frac2{n^2}\int_1^\infty t e^{-t}\,dt
\lesssim
\frac1n.
\end{align*}

The part arising from the fixed \(n^{-1}\) failure probability satisfies
\begin{align*}
&\frac1n
\int_1^{cn}
\left(
\frac1n+
\frac{3\sqrt t}{n^{3/2}}+
\frac{2t}{n^2}
\right)\,dt
=
\frac1n
\left[
\frac{cn-1}{n}
+
\frac{2\{(cn)^{3/2}-1\}}{n^{3/2}}
+
\frac{(cn)^2-1}{n^2}
\right]
\lesssim
\frac1n.
\end{align*}
Also, notice that $e^{-cn}\lesssim\frac1n$,
combining these bounds together, we have
\begin{align}
\mathbb E_{\mathbb P^{2n}}
\left[
\left(
\left|
\widehat\beta-\beta(\mathbb P)
\right|
-C_0\bar\Gamma
\right)_+^2
\right]
\lesssim
\frac1n.
\label{eq:positive-part-second-moment}
\end{align}

For every nonnegative number \(v\) and every \(d\ge0\), we have $v
\le
d+(v-d)_+$.

Applying this with $v=
\left|
\widehat\beta-\beta(\mathbb P)
\right|
\ \text{and}\ 
d=C_0\bar\Gamma$ gives
\[
\left|
\widehat\beta-\beta(\mathbb P)
\right|
\le
C_0\bar\Gamma+
\left(
\left|
\widehat\beta-\beta(\mathbb P)
\right|
-C_0\bar\Gamma
\right)_+.
\]
Therefore, we have
\begin{align*}
\left|
\widehat\beta-\beta(\mathbb P)
\right|^2
&\le
2C_0^2\bar\Gamma^2
+
2\left(
\left|
\widehat\beta-\beta(\mathbb P)
\right|
-C_0\bar\Gamma
\right)_+^2.
\end{align*}
Taking expectations and applying
\eqref{eq:positive-part-second-moment}, we obtain
\[
\mathbb E_{\mathbb P^{2n}}
\left[
\left|
\widehat\beta-\beta(\mathbb P)
\right|^2
\right]
\lesssim
\bar\Gamma^2+\frac1n.
\]

The constants in this bound are uniform over every admissible
\(\mathcal F_\mu,\mathcal F_\pi\) and every
\(\mathbb P\in\mathcal P(\mathcal F_\mu,\mathcal F_\pi)\). Hence, we have
\begin{align*}
&\sup_{\substack{
\mathcal F_\mu\in\mathcal H_n(a_\mu,s_\mu)\\
\mathcal F_\pi\in\mathcal H_n(a_\pi,s_\pi)
}}
\inf_{\widetilde\beta}
\sup_{\mathbb P\in
\mathcal P(\mathcal F_\mu,\mathcal F_\pi)}
\mathbb E_{\mathbb P^{2n}}
\left[
\left|
\widetilde\beta-\beta(\mathbb P)
\right|^2
\right]
\lesssim
\bar\Gamma^2+\frac1n.
\end{align*}

Since we already proved that
\[
\bar\Gamma
\le
8(B+\rho_n^2).
\]
Therefore, we have
\begin{align*}
\bar\Gamma^2
\le
64(B+\rho_n^2)^2
\le
128B^2+128\rho_n^4.
\end{align*}
Because $\rho_n^4
=
\frac{(\log n)^2}{n^2}
\lesssim
\frac1n$, it follows that
\[
\bar\Gamma^2+\frac1n
\lesssim
B^2+\frac1n.
\]
We have therefore proven
\[
\mathcal E
\left(
n,\{a_q,s_q\}_{q\in\{\mu,\pi\}}
\right)
\lesssim
\frac1n+B^2.
\]

Furthermore, since $\left|
\widehat\beta-\beta(\mathbb P)
\right|^2\le36$, we obtain that
\[
\mathcal E
\left(
n,\{a_q,s_q\}_{q\in\{\mu,\pi\}}
\right)
\lesssim
1\wedge
\left\{
\frac1n+B^2
\right\}.
\]
Finally, according to Theorem~\ref{thm:lower_bound}, we have
\[
\mathcal E
\left(
n,\{a_q,s_q\}_{q\in\{\mu,\pi\}}
\right)
\gtrsim
1\wedge
\left\{
\frac1n+B^2
\right\}.
\]
Combining the upper and lower bounds proves
\[
\mathcal E
\left(
n,\{a_q,s_q\}_{q\in\{\mu,\pi\}}
\right)
\asymp
1\wedge
\left\{
\frac1n+
\left(
a_\mu a_\pi+
\min\left\{
a_\pi s_\mu+s_\pi^2,\,
a_\mu s_\pi+s_\mu^2
\right\}
\right)^2
\right\}.
\]
We finish the proof.
\end{proof}
\section{Omitted Proofs in Appendix \ref{app:proofs_thms}}
\begin{proof}[Proof of Lemma \ref{lem:Rad_com_G(t,M)}]
Fixing \(t\geq0\), \(M\geq1\), and \(\delta>0\), we  write
\[
g_0:=0,
\ 
g_j:=t\phi_j,\  1\leq j\leq M,
\]
so that $\mathcal G(t,M)=\{g_0,g_1,\ldots,g_M\}$.
For \(0\leq j,k\leq M\), define $h_{jk}:=g_j-g_k$.

Every member of the difference class has this form.   Therefore, we have
\[
\partial\mathcal G(t,M)
=
\{h_{jk}:0\leq j,k\leq M\},
\ 
|\partial\mathcal G(t,M)|\leq(M+1)^2.
\]
The inequality allows for the possibility that two different ordered pairs
produce the same function.  Moreover, since \(|g_j(x)|\leq t\) for every
\(j\) and every \(x\),
\begin{equation}
|h_{jk}(x)|
=
|g_j(x)-g_k(x)|
\leq2t.
\label{eq:code-difference-envelope}
\end{equation}

Let $\mathcal W_\delta
:=
\left\{
h\in\partial\mathcal G(t,M):\|h\|_2\leq\delta
\right\}$. The class \(\mathcal W_\delta\) contains zero and is symmetric: if
\(h_{jk}\in\mathcal W_\delta\), then
\(-h_{jk}=h_{kj}\in\mathcal W_\delta\).  Thus, for every realization of
\(X_{1:n}\) and \(\xi_{1:n}\), we have
\[
\sup_{h\in\mathcal W_\delta}
\frac1n\sum_{i=1}^n\xi_i h(X_i)
=
\sup_{h\in\mathcal W_\delta}
\left|
\frac1n\sum_{i=1}^n\xi_i h(X_i)
\right|.
\]
In particular, the absence of an absolute value in the definition of
\(\mathcal R_n\) does not change its value for this localized difference
class.  Discarding the localization constraint can only increase the
supremum, so
\begin{align}
\mathcal R_n\bigl(\delta;\partial\mathcal G(t,M)\bigr)
&=
\mathbb E_{X_{1:n},\xi_{1:n}}
\left[
\sup_{h\in\mathcal W_\delta}
\frac1n\sum_{i=1}^n\xi_i h(X_i)
\right]\leq
\mathbb E_{X_{1:n},\xi_{1:n}}
\left[
\max_{0\leq j,k\leq M}
\frac1n\sum_{i=1}^n\xi_i h_{jk}(X_i)
\right].
\label{eq:localized-to-global-code}
\end{align}

Conditioning on \(X_{1:n}\), we set
\[
Z_{jk}
:=
\frac1n\sum_{i=1}^n\xi_i h_{jk}(X_i).
\]
For every \(\lambda>0\), by the independence of the Rademacher signs and the
inequality \(\cosh(a)\leq e^{a^2/2}\), we obtain
\begin{align}
\mathbb E_{\xi_{1:n}}
\left[
e^{\lambda Z_{jk}}
\mid X_{1:n}
\right]=
\prod_{i=1}^n
\cosh\left(
\frac{\lambda h_{jk}(X_i)}{n}
\right)\leq
\exp\left\{
\frac{\lambda^2}{2n^2}
\sum_{i=1}^n h_{jk}(X_i)^2
\right\}\leq
\exp\left\{
\frac{2\lambda^2t^2}{n}
\right\}.
\label{eq:code-difference-mgf}
\end{align}
The last inequality follows from
\eqref{eq:code-difference-envelope}, which implies that $\sum_{i=1}^n h_{jk}(X_i)^2\leq4nt^2$.

Applying Jensen's inequality, using the elementary bound $e^{\lambda\max_{j,k}Z_{jk}}
\leq
\sum_{j=0}^M\sum_{k=0}^M e^{\lambda Z_{jk}}$
and \eqref{eq:code-difference-mgf}, we have that
\begin{align*}
\exp\left\{
\lambda
\mathbb E_{\xi_{1:n}}
\left[
\max_{0\leq j,k\leq M}Z_{jk}
\mid X_{1:n}
\right]
\right\}
&\leq
\mathbb E_{\xi_{1:n}}
\left[
e^{\lambda\max_{j,k}Z_{jk}}
\mid X_{1:n}
\right]
\leq
\sum_{j=0}^M\sum_{k=0}^M
\mathbb E_{\xi_{1:n}}
\left[
e^{\lambda Z_{jk}}
\mid X_{1:n}
\right]\\
&\leq
(M+1)^2
\exp\left\{
\frac{2\lambda^2t^2}{n}
\right\}.
\end{align*}
Taking logarithms and dividing by \(\lambda\) on both sides yields
\begin{equation}
\mathbb E_{\xi_{1:n}}
\left[
\max_{0\leq j,k\leq M}Z_{jk}
\mid X_{1:n}
\right]
\leq
\frac{2\log(M+1)}{\lambda}
+
\frac{2\lambda t^2}{n}.
\label{eq:code-max-before-optimization}
\end{equation}
If \(t=0\), every function in \(\partial\mathcal G(0,M)\) is zero and the
claim is immediate.  Suppose that \(t>0\).  Since \(M\geq1\), we may choose
\[
\lambda
:=
\frac{\sqrt{n\log(M+1)}}{t}.
\]
Substituting this choice into \eqref{eq:code-max-before-optimization}, we obtain
\begin{align*}
\mathbb E_{\xi_{1:n}}
\left[
\max_{0\leq j,k\leq M}Z_{jk}
\mid X_{1:n}
\right]
\leq
2t\sqrt{\frac{\log(M+1)}{n}}
+
2t\sqrt{\frac{\log(M+1)}{n}}
=
4t\sqrt{\frac{\log(M+1)}{n}}.
\end{align*}
The right-hand side is independent of \(X_{1:n}\).  Taking expectation over
\(X_{1:n}\) in \eqref{eq:localized-to-global-code} proves
\[
\mathcal R_n\bigl(\delta;\partial\mathcal G(t,M)\bigr)
\leq
4t\sqrt{\frac{\log(M+1)}{n}}.
\]
\end{proof}

\begin{proof}[Proof of Lemma \ref{lem:chi2_dist}]
Recall from the construction that \(b=\alpha r\).  Define the
one-observation likelihood ratios with respect to \(Q\) by
\[
L_*(x,t,y)
:=
\frac{dP_*}{dQ}(x,t,y)
=
1+bty
\]
and
\[
L_j(x,t,y)
:=
\frac{dP_j}{dQ}(x,t,y)
=
\{1+\alpha t\phi_j(x)\}
\{1+ry\phi_j(x)\},
\ 1\leq j\leq M.
\]
Because \(0\leq\alpha,r\leq1/8\), we have
\(0\leq b\leq1/64\), and hence,
\[
L_*(x,t,y)=1+bty\geq1-b>0.
\]
Thus, every \(P_j\) is absolutely continuous with respect to \(P_*\).

We first calculate the one-observation likelihood cross moments.  For
\(1\leq j,k\leq M\), define
\[
A_{jk}
:=
\mathbb E_{P_*}
\left[
\frac{dP_j}{dP_*}(X,T,Y)
\frac{dP_k}{dP_*}(X,T,Y)
\right].
\]
Since \(dP_*=L_*\,dQ\) and \(dP_j/dP_*=L_j/L_*\),
\begin{equation}
A_{jk}
=
\mathbb E_Q
\left[
\frac{L_j(X,T,Y)L_k(X,T,Y)}
     {L_*(X,T,Y)}
\right].
\label{eq:one-observation-cross-moment}
\end{equation}

Define $S:=\alpha T+rY,
\ 
U:=TY$. Because \(\phi_j(X)^2=1\) and \(b=\alpha r\), we have
\begin{align}
L_j(X,T,Y)
&=
\{1+\alpha T\phi_j(X)\}
\{1+rY\phi_j(X)\}
\nonumber\\
&=
1+\alpha T\phi_j(X)+rY\phi_j(X)
+\alpha rTY\phi_j(X)^2
\nonumber\\
&=
1+bTY+\phi_j(X)(\alpha T+rY)
\nonumber\\
&=
L_*(X,T,Y)+\phi_j(X)S.
\label{eq:likelihood-decomposition}
\end{align}
The same identity holds with \(j\) replaced by \(k\).  Therefore, we obtain
\begin{align}
\frac{L_jL_k}{L_*}
=
\frac{
\{L_*+\phi_j(X)S\}
\{L_*+\phi_k(X)S\}
}{L_*}
=
L_*
+\{\phi_j(X)+\phi_k(X)\}S
+\phi_j(X)\phi_k(X)\frac{S^2}{L_*}.
\label{eq:cross-integrand-expansion}
\end{align}

Under \(Q\), \(X\) is independent of \((T,Y)\).  Moreover, \(T\) and \(Y\)
are independent Rademacher random variables.  Consequently,
\[
\mathbb E_Q[L_*]
=
\mathbb E_Q[1+bTY]
=1,
\ 
\mathbb E_Q[S]
=
\alpha\mathbb E_Q[T]+r\mathbb E_Q[Y]
=0.
\]
Taking expectations in \eqref{eq:cross-integrand-expansion} and using the
independence of \(X\) and \((T,Y)\), we have
\begin{equation}
A_{jk}
=
1+
\mathbb E_X[\phi_j(X)\phi_k(X)]
\mathbb E_{T,Y}
\left[
\frac{S^2}{1+bU}
\right].
\label{eq:cross-moment-reduced}
\end{equation}

It remains to evaluate the second expectation.  Because \(T\) and \(Y\)
are independent Rademacher variables, \(U=TY\) is also Rademacher, so $\mathbb E[U]=0,
\ U^2=1$.
Moreover, notice that
\[
S^2
=
(\alpha T+rY)^2
=
\alpha^2+r^2+2\alpha rTY
=
\alpha^2+r^2+2bU,\ \frac1{1+bU}
=
\frac{1-bU}{1-b^2}.
\]
It follows that
\begin{align}
\mathbb E_{T,Y}
\left[
\frac{S^2}{1+bU}
\right]
&=
\frac1{1-b^2}
\mathbb E
\left[
\{\alpha^2+r^2+2bU\}(1-bU)
\right]
\nonumber\\
&=
\frac1{1-b^2}
\mathbb E
\left[
\alpha^2+r^2
+\{2b-b(\alpha^2+r^2)\}U
-2b^2U^2
\right]
\nonumber\\
&=
\frac{\alpha^2+r^2-2b^2}{1-b^2}
\nonumber\\
&=
D.
\label{eq:diagonal-increment}
\end{align}
Combining \eqref{eq:cross-moment-reduced} and
\eqref{eq:diagonal-increment}, we obtain that
\[
A_{jk}
=
1+D\mathbb E_X[\phi_j(X)\phi_k(X)].
\]
The code functions are orthonormal in \(L_2(\nu_d)\), so we have $\mathbb E_X[\phi_j(X)\phi_k(X)]
=
\mathbf 1\{j=k\}$.

Consequently,
\begin{equation}
A_{jk}
=
\begin{cases}
1+D, & j=k,\\
1,   & j\neq k.
\end{cases}
\label{eq:cross-moment-final}
\end{equation}
In particular, every off-diagonal cross moment is exactly one.

We now pass to \(N\) observations.  Write
\(Z_i=(X_i,T_i,Y_i)\).  By the product structure,
\[
\frac{dP_j^{\otimes N}}{dP_*^{\otimes N}}(Z_1,\ldots,Z_N)
=
\prod_{i=1}^N\frac{L_j(Z_i)}{L_*(Z_i)}.
\]
Hence,
\begin{equation}
\frac{d\overline P_N}{dP_*^{\otimes N}}(Z_1,\ldots,Z_N)
=
\frac1M\sum_{j=1}^M
\prod_{i=1}^N\frac{L_j(Z_i)}{L_*(Z_i)}.
\label{eq:mixture-likelihood-ratio}
\end{equation}
Recall that
\[
\chi^2(P\|Q)
:=
\mathbb E_Q
\left[
\left(\frac{dP}{dQ}-1\right)^2
\right]
=
\mathbb E_Q
\left[
\left(\frac{dP}{dQ}\right)^2
\right]-1.
\]
Using \eqref{eq:mixture-likelihood-ratio}, expanding the square, and using
independence under \(P_*^{\otimes N}\), we obtain
\begin{align}
1+
\chi^2\!\left(\overline P_N\middle\|P_*^{\otimes N}\right)
&=
\mathbb E_{P_*^{\otimes N}}
\left[
\left(
\frac{d\overline P_N}{dP_*^{\otimes N}}
\right)^2
\right] \nonumber\\
&=
\frac1{M^2}
\sum_{j=1}^M\sum_{k=1}^M
\mathbb E_{P_*^{\otimes N}}
\left[
\prod_{i=1}^N
\frac{L_j(Z_i)L_k(Z_i)}{L_*(Z_i)^2}
\right]
\nonumber\\
&=
\frac1{M^2}
\sum_{j=1}^M\sum_{k=1}^M
\left\{
\mathbb E_{P_*}
\left[
\frac{L_j(Z)L_k(Z)}{L_*(Z)^2}
\right]
\right\}^{N}
\nonumber\\
&=
\frac1{M^2}
\sum_{j=1}^M\sum_{k=1}^M A_{jk}^N.
\label{eq:mixture-second-moment}
\end{align}
By \eqref{eq:cross-moment-final}, the last double sum contains \(M\)
diagonal terms equal to \((1+D)^N\) and \(M(M-1)\) off-diagonal terms
equal to one.  Thus, we have
\begin{align*}
1+
\chi^2\!\left(\overline P_N\middle\|P_*^{\otimes N}\right)
&=
\frac{M(1+D)^N+M(M-1)}{M^2}
=
1+\frac{(1+D)^N-1}{M}.
\end{align*}
Subtracting one from both sides, we prove that
\[
\chi^2\!\left(\overline P_N\middle\|P_*^{\otimes N}\right)
=
\frac{(1+D)^N-1}{M}.
\]
Finally, by definition, we know that
\[
d_{\mathrm{TV}}(P,Q)
:=
\frac12
\mathbb E_Q
\left|
\frac{dP}{dQ}-1
\right|.
\]
By the Cauchy-Schwarz inequality, we have
\begin{align*}
d_{\mathrm{TV}}\!\left(\overline P_N,P_*^{\otimes N}\right)
&=
\frac12
\mathbb E_{P_*^{\otimes N}}
\left|
\frac{d\overline P_N}{dP_*^{\otimes N}}-1
\right|
\leq
\frac12
\left\{
\mathbb E_{P_*^{\otimes N}}
\left[
\left(
\frac{d\overline P_N}{dP_*^{\otimes N}}-1
\right)^2
\right]
\right\}^{1/2}
=
\frac12
\sqrt{
\chi^2\!\left(\overline P_N\middle\|P_*^{\otimes N}\right)
}
\end{align*}
Thus, we prove that
\[
d_{\mathrm{TV}}\!\left(\overline P_N,P_*^{\otimes N}\right)=
\frac12
\sqrt{\frac{(1+D)^N-1}{M}}.
\]
This proves the second claim.
\end{proof}

\begin{proof}[Proof of Lemma \ref{lem:two-sided-truncation}]
Set $L:=\min\{u+\tau,v+\rho\}$.
The values of \(x=u\wedge\rho\) and \(y=v\wedge\tau\) depend on whether
\(u\) and \(v\) lie below or above their respective truncation levels.  We
therefore consider the four possible configurations.

\medskip
\noindent
\emph{Case 1: \(u\leq\rho\) and \(v\leq\tau\).}
Here \(x=u\) and \(y=v\).  If \(\rho\leq\tau\), then \(z=\rho\), and
\[
L
\leq
v+\rho
=
y+z.
\]
If instead \(\tau<\rho\), then \(z=\tau\), and
\[
L
\leq
u+\tau
=
x+z.
\]

\medskip
\noindent
\emph{Case 2: \(u\leq\rho\) and \(v>\tau\).}
Now \(x=u\) and \(y=\tau\).  Hence,
\[
L
\leq
u+\tau
=
x+y.
\]

\medskip
\noindent
\emph{Case 3: \(u>\rho\) and \(v\leq\tau\).}
Now \(x=\rho\) and \(y=v\).  Therefore,
\[
L
\leq
v+\rho
=
x+y.
\]

\medskip
\noindent
\emph{Case 4: \(u>\rho\) and \(v>\tau\).}
This case forces \(\rho>0\) and \(\tau>0\).  Indeed, if either were zero,
then \(u>0\) and \(v>0\), so \(uv>0\), whereas
\(p\sqrt{\rho\tau}=0\), contradicting
\(uv\leq p\sqrt{\rho\tau}\).

Suppose first that \(\rho\leq\tau\).  Since \(v>\tau\), we have $v>\tau\geq\sqrt{\rho\tau}$.

Together with the assumed product inequality, this yields
\[
uv
\leq
p\sqrt{\rho\tau}
<
pv.
\]
Since \(v>0\), division by \(v\) gives \(u<p\).  In this subcase
\(y=\tau\), and hence
\[
L
\leq
u+\tau
\leq
p+y.
\]

Suppose instead that \(\tau<\rho\).  Since \(u>\rho\), we have
\[
u>\rho>\sqrt{\rho\tau}.
\]
It follows that $uv
\leq
p\sqrt{\rho\tau}
<
pu$.

Since \(u>0\), division by \(u\) gives \(v<p\).  In this subcase
\(x=\rho\), and therefore
\[
L
\leq
v+\rho
\leq
p+x.
\]

All four cases establish an upper bound by a sub-sum of
\(p+x+y+z\).  Therefore,
\[
\min\{u+\tau,v+\rho\}
=L
\leq
p+x+y+z.
\]
Adding \(p\) to both sides gives
\[
p+\min\{u+\tau,v+\rho\}
\leq
2p+x+y+z.
\]
Finally, if \(m:=\max\{p,x,y,z\}\), then we have
$2p+x+y+z
\leq
2m+m+m+m
=
5m$.

Thus, we finish the proof.
\end{proof}


\end{document}